%% file: arxiv.tex
\documentclass{article}
\usepackage[utf8]{inputenc}
\usepackage{Neurips/macros}
\usepackage[numbers]{natbib}

\usepackage{geometry}

\newcommand{\enayat}[1]{}
\newcommand{\mnote}[1]{}
\newcommand{\rnote}[1]{}
\newcommand{\cnote}[1]{}
\newcommand{\ranote}[1]{}
\newcommand{\tomas}[1]{}
\newcommand{\todo}[1]{}

\title{Faster Rates of Convergence to Stationary Points in Differentially Private Optimization}

\author{
\makebox[1.5in]{\hfill Raman Arora\thanks{Department of Computer Science, The Johns Hopkins University,
 \href{mailto:arora@cs.jhu.edu}{arora@cs.jhu.edu} }}
\makebox[1.6in]{\hfill Raef Bassily\thanks{Department of Computer Science \& Engineering and the Translational Data Analytics Institute (TDAI), The Ohio State University,  \href{mailto:bassily.1@osu.edu}{bassily.1@osu.edu} }}
\makebox[1.8in]{\hfill Tom\'as Gonz\'alez
\thanks{Institute for Mathematical and Computational Engineering, Pontificia Universidad Cat\'olica de Chile, \href{mailto:tsgonzalez@uc.cl}{tsgonzalez@uc.cl}}}
\and
\makebox[1.8in]{\hfill Crist\'obal Guzm\'an
\thanks{Institute for Mathematical and Computational Engineering, Pontificia Universidad Cat\'olica de Chile,
 \href{mailto:crguzmanp@mat.uc.cl}{crguzmanp@mat.uc.cl} }}
\makebox[1.5in]{\hfill Michael Menart \thanks{Department of Computer Science \& Engineering, The Ohio State University,
\href{mailto:menart.2@osu.edu}{menart.2@osu.edu}}}
\makebox[1.5in]{ \hfill Enayat Ullah\thanks{Department of Computer Science, The Johns Hopkins University, \href{mailto:enayat@jhu.edu}{enayat@jhu.edu}}}}

\begin{document}

\date{}

\maketitle

\begin{abstract}
We study the problem of approximating stationary points of Lipschitz and smooth functions under $(\varepsilon,\delta)$-differential privacy (DP) in both the finite-sum and stochastic settings. A point $\widehat{w}$ is called an $\alpha$-stationary point of a function $F:\mathbb{R}^d\rightarrow\mathbb{R}$ if $\|\nabla F(\widehat{w})\|\leq \alpha$. We provide a new efficient algorithm that finds an $\tilde{O}\big(\big[\frac{\sqrt{d}}{n\varepsilon}\big]^{2/3}\big)$-stationary point in the finite-sum setting, where $n$ is the number of samples. This improves on the previous best rate of $\tilde{O}\big(\big[\frac{\sqrt{d}}{n\varepsilon}\big]^{1/2}\big)$. We also give a new construction that improves over the existing rates in the stochastic optimization setting, where the goal is to find approximate stationary points of the population risk. Our construction finds a $\tilde{O}\big(\frac{1}{n^{1/3}} + \big[\frac{\sqrt{d}}{n\varepsilon}\big]^{1/2}\big)$-stationary point of the population risk in time linear in $n$. Furthermore, under the additional assumption of convexity, we completely characterize the sample complexity of finding stationary points of the population risk (up to polylog factors) and show that the optimal rate on population stationarity is $\tilde \Theta\big(\frac{1}{\sqrt{n}}+\frac{\sqrt{d}}{n\varepsilon}\big)$. Finally, we show that our methods can be used to provide dimension-independent rates of $O\big(\frac{1}{\sqrt{n}}+\min\big(\big[\frac{\sqrt{\rank}}{n\varepsilon}\big]^{2/3},\frac{1}{(n\varepsilon)^{2/5}}\big)\big)$ on population stationarity for Generalized Linear Models (GLM), where $\rank$ is the rank of the design matrix, which improves upon the previous best known rate.
\end{abstract}

\section{Introduction}
Protecting users' data in machine learning models has become a central concern in multiple contexts, e.g.~those involving financial or health data. In this respect,
differential privacy (DP) is the gold standard for rigorous privacy protection \cite{DR14}. Therefore, recent 
research has focused on the limits and possibilities of solving some of the most well-established machine learning problems under the constraint of DP. Despite intensive research, some fundamental problems
remain not completely understood. One example is nonconvex optimization; namely, the task of approximating stationary points, which has been heavily studied in recent years in the non-private setting \cite{fang2018spider,ma_implicitregularization,carmon_convexguilty,Nesterov2006CubicRO,Ghadimi:2013,arjevani2019lower,foster2019complexity}. This problem is motivated by the intractability of nonconvex (global) optimization, as well as 
by a number of settings where stationary points have been shown to be global minima  \cite{ge_matrixcompletion,sun_phaseretrieval}.
\subsection{Contributions}

In this work, we make progress towards resolving the complexity of approximating stationary points in optimization under the constraint of differential privacy, for both empirical and population risks. A summary of our new results is available in Table \ref{tab:results_summary}.
In what follows, $d$ is the problem dimension, $n$ is the dataset size, and $\varepsilon,\delta$ are the approximate DP parameters.  %
Our first set of results pertains to the approximation of stationary points in empirical nonconvex optimization (a.k.a.~finite-sum case). %
In this context, we provide algorithms with rate $O\big( \big[\frac{\sqrt{d}}{n\varepsilon}\big]^{2/3} \big)$, and oracle complexity\footnote{We consider for complexity the first-order oracle model, standard for continuous optimization \cite{nemirovskij1983problem}.} $\tilde O\big( \max\big\{ \big(\frac{n^5\varepsilon^2}{d}\big)^{1/3}, \big(\frac{n\varepsilon}{\sqrt{d}}\big)^{2} \big\}\big)$. This rate is sharper than the best known for this problem \cite{wang2017differentially}. 

Next, we focus on the task of approximating stationary points of the population risk. Results for this problem are scarce. We provide the fastest rate up to date for this problem under DP, of $\tilde O\big(\frac{1}{n^{1/3}}+\big[\frac{\sqrt{d}}{n\varepsilon }\big]^{1/2}\big)$, with an algorithm that moreover has oracle complexity $n$ (i.e., is single-pass). This algorithm is a noisy version of the SPIDER algorithm \cite{fang2018spider}, whose gradient estimators are built using a tree-aggregation data structure for prefix-sums \cite{asi2021private}.

We continue by investigating stationary points for convex losses and give an algorithm based on the recursive regularization technique of \cite{allen2018make} which achieves the optimal rate of $\tilde \Theta\big(\frac{1}{\sqrt{n}}+\frac{\sqrt{d}}{n\varepsilon} \big)$ on population stationarity.
To establish optimality, we give a lower bound of $\Omega\big(\frac{\sqrt{d}}{n\varepsilon}\big)$ on empirical stationarity under DP (Theorem \ref{thm:lower-bound-convex}) and a non-private lower bound of $\Omega(\frac{1}{\sqrt{n}})$ on population stationarity (Theorem \ref{thm:sample-complexity-lower-bound}). 
We also give a linear-time method, which achieves the optimal rate when the smoothness parameter is not so large.
We conclude the paper showing a black-box reduction that converts any DP method for finding stationary points of smooth and Lipschitz losses into a DP method with {\em dimension-independent rates} for the case of generalized linear models (GLM). 
Using our proposed method with Private Spiderboost as the base algorithm yields a rate of $\tilde O\br{\frac{1}{\sqrt{n}}+\min\br{\big[
\frac{\sqrt{\rank}}{n\varepsilon}\big]^{2/3},\frac{1}{(n\varepsilon)^{2/5}}}}$ %
on population stationarity. This improves upon the result of \cite{song2021evading} 
which proposed a method with $\tilde O\big(\big[\frac{\sqrt{\rank}}{n\varepsilon}\big]^{1/2}\big)$ %
empirical stationarity\footnote{
This is the rate obtained after fixing a mistake in the proof of Theorem 4.1 in
\cite{song2021evading}. 
}.

 \input{Arxiv/arxiv-table}
\subsection{Our Techniques}

Our methods combine multiple techniques from optimization and differential privacy in novel ways.  The lower bound for the empirical norm of the gradient uses fingerprinting codes to a loss similar to that used for Differentially Private-Empirical Risk Minimization (DP-ERM) \cite{bassily2014private}, crafted to work in the unconstrained case. This lower bound can be extended to the population gradient norm by a known re-sampling argument \cite{BFTT19}.  %
We also give a non-private lower bound of $\Omega\br{1/\sqrt{n}}$ on population stationarity with $n$ samples which holds even in dimension 1, as opposed to previous results \cite{foster2019complexity}.

Efficient algorithms for (both empirical and population) norm of the gradient are derived using noisy versions of variance-reduced stochastic first order methods, which have proved remarkably useful in DP stochastic optimization \cite{asi2021private,bassily2021non, bassily2021differentially}. 
However, in contrast to previous work which scales noise proportionally to the Lipschitz constant \cite{ZCHSB20,zhang2021private} or (in the case of constrained optimization) the diameter of the constraint set \cite{bassily2021differentially,bassily2021non},  %
we observe that the gradient variations between %
iterates $w,w'$ can be privatized more effectively by scaling the noise proportional to $\smooth\norm{w-w'}$.
In the case of the empirical risk, we use a noisy version of SpiderBoost \cite{spiderboost}. We remark that our methods can achieve comparable rates when applied to similar algorithms such as Spider \cite{fang2018spider} and Storm \cite{CO19}, but SpiderBoost allows for a larger learning rate which is considered better in practice. For the population risk, it is worth noting that the empirical norm of the gradient does not translate directly into population gradient guarantees, even if the algorithm in use is uniformly stable \cite{bousquet2002stability}, since this type of guarantee does not enjoy a {\em stability-implies-generalization} property. Therefore, we opt for single pass methods that combine variance-reduction with tree-aggregation; these techniques are particularly suitable for the classical Spider algorithm \cite{fang2018spider}, which is the one we base our method on. For the convex setting, we use recursive regularization \cite{allen2018make} which was used to achieve the optimal non-private rate by \cite{foster2019complexity}.

Finally, our method for (non-convex) GLMs uses the Johnson-Lindenstrauss based dimensionality reduction technique similar to 
\cite{arora2022differentially}, which focused on the convex setting.
Moreover, for population stationarity of GLMs, we give a new uniform convergence result of gradients of Lipschitz functions. 
This guarantee, unlike the prior work of \cite{foster18}, has only poly-logarithmic dependence on the radius of the constraint set, which is crucial for our analysis.

\subsection{Related Work}
The current work fits within the literature of differentially private optimization, which has primarily focused on the convex case \cite{CMS,jain2012differentially,kifer2012private, bassily2014private,talwar2014private,JTOpt13, TTZ15a, BFTT19, feldman2020private, asi2021private, bassily2021non}. The culmination of this line of work for the convex smooth case showed that optimal rates are achievable in linear time \cite{feldman2020private,asi2021private, bassily2021non}. Our work shows that in the convex case similar rates are achievable for the norm of the gradient: this result is useful, e.g.,~for dual formulations of linearly constrained  convex programs \cite{nesterov2012make}, and moreover it has become a problem of independent interest \cite{allen2018make,foster2019complexity}. \footnote{To provide a specific example, consider the dual of the regularized discrete optimal transport problem, as discussed in \cite{diakonikolas2023complementary}, Section 5.6. If the marginals $\mu,\nu$ in that model are accessed through i.i.d. samples, then this becomes an SCO problem. Moreover, it is argued in that reference that approximate stationary points provide approximately feasible and optimal transports through duality arguments. Hence, the result is an SCO problem where we require {\em approximate stationary points.}}

Regarding stationary points for nonconvex losses, work in DP is far more recent, and primarily focused on the empirical stationarity \cite{wang2017differentially,zhang2017efficient, wang2019differentially, WCX19}. 
Under similar assumptions to ours these works approximate stationary points with rate $\tilde{O}\big(\big[\frac{\sqrt{d}}{n\varepsilon}\big]^{1/2}\big)$, which is slower %
than ours. %

Works addressing population guarantees for the norm of the gradient under DP are scarce. \cite{ZCHSB20} proposed a noisy gradient method, whose population guarantee is obtained by generalization properties of DP. 
However, the best guarantee obtainable with their analysis is 
$O\big(\big[ \frac{\sqrt{d}}{n\varepsilon} \big]^{1/2} + \sqrt{d}\varepsilon \big)$\footnote{\cite{ZCHSB20} omits the term $\sqrt{d}\varepsilon$, but this omission is only valid when $\varepsilon<1/[n\sqrt{d}]^{1/3}$. 
}. 
Note that for any $\varepsilon$ this rate is $\Omega\big([d/n]^{1/3}\big)$. Under additional assumptions (on the Hessian), \cite{wang2019differentially} obtains a rate of $\tilde O(\sqrt{d/(n\varepsilon)})$ by uniform convergence of gradients, which is sharper when $\varepsilon$ is constant.
By contrast, our rate is much faster than both for $\varepsilon=\Theta(1)$.
In particular, in this range, our rates are faster than those obtained by uniform convergence, $O(\sqrt{d/n})$ \cite{foster18}. 
Moreover, our method runs in time linear in~$n$. On the other hand, in the much more restrictive setting where the loss satisfies the Polyak-{\L}ojasiewicz (PL) inequality,  %
\cite{zhang2021private} provide \emph{population risk} bounds of $\tilde{O}(d/[n\varepsilon]^2)$ under DP.

    The work of \cite{bassily2021differentially} studies population guarantees for stationarity in constrained settings, obtaining rates $O\big(\frac{1}{n^{1/3}}+\big[\frac{\sqrt d}{n\varepsilon}\big]^{2/5}\big)$ in linear time. Notice first that these guarantees are based on the Frank-Wolfe gap, making those results incomparable to ours. Despite this fact,  their rates are slower than ours.\footnote{We believe our methods can be extended to constrained settings 
    using
    gradient mapping,
    a guarantee for which
    is stronger than for Frank-Wolfe gap \cite[Section 7.5.1]{lan2020first}. %
    We defer this extension to future work.} On the other hand, they provide results for (close to nearly) stationary points in constrained/unconstrained settings, for a broader class of {\em weakly convex losses} (possibly nonsmooth). This result is then more general, but the rate of  $O\big(\frac{1}{n^{1/4}}+\big[\frac{\sqrt d}{n\varepsilon}\big]^{1/3}\big)$ is substantially slower than ours, and their algorithm has oracle complexity which is superlinear in $n$.

    The problem of stationary points in (nonprivate) stochastic optimization has drawn major attention recently \cite{Ghadimi:2013,Ghadimi:2016,fang2018spider,allen2018make,foster18,foster2019complexity,arjevani2019lower}. To the best of our knowledge, no lower bounds for the sample complexity\footnote{Sample complexity is the fundamental limit on the sample size needed, as a function of $\alpha$, to achieve $\alpha$ stationarity. This is different from the oracle complexity as one is not limited to first-order methods.} of this problem are known (beyond those known for the convex case \cite{foster2019complexity}). 
    On the other hand, oracle complexity is by now understood: in high dimensions, for (on average) smooth losses the optimal stochastic oracle complexity rate is $O(1/n^{1/3})$ \cite{arjevani2019lower}. Although this provides some evidence of the sharpness of our results (see Appendix \ref{app:rate-challenges}), note that these lower bounds require very high dimensional constructions (namely, $d=\Omega(1/\alpha^4)$, 
    where $\alpha$ is the rate), which limits their applicability in the private setting.

\input{sections/prelim}
\input{Arxiv/arxiv-emp-sp}
\input{Arxiv/arxiv-pop-sp-v2}

\section{Stationary Points in the Convex Setting}
\input{Arxiv/arxiv-convex}

\input{sections/glm}

\section*{Acknowledgements}\label{sec:ack}
RA and EU are supported, in part, by NSF BIGDATA award IIS-1838139 and NSF CAREER award IIS-1943251. RB's and MM's research is supported by NSF CAREER Award 2144532 and NSF Award AF-1908281. 
CG and TG's research was partially supported by INRIA Associate Teams project, FONDECYT 1210362 grant, ANID Anillo ACT210005 grant, and National Center for Artificial Intelligence CENIA FB210017, Basal ANID. 

\newpage
\appendix

\bibliography{references}
\bibliographystyle{alpha}

\newpage

\input{appendix/app-lower-bounds}

\input{appendix/app-emp-sp}

\input{appendix/app-pop-sp}

\input{appendix/app-convex}

\input{appendix/app-glm}

\end{document}

%% file: Arxiv/arxiv-table.tex
{\renewcommand{\arraystretch}{1.8}
\begin{table}[t]
    \centering
    \begin{tabular}{|c|c|c c|c p{0.08\linewidth}|}
    \hline
          {\small Setting} &  {\small Convergence} & \multicolumn{2}{c}{{\small Our Rate}}  & \multicolumn{2}{|c|}{{\small Previous best-known rate}} \\
            \hline
      \hline
        \multirow{3}{*}{{\small Non-convex}} & 
         {\small Empirical} & $\br{\frac{\sqrt{d}}{n\varepsilon}}^{2/3}$ \ \ & {\small (Thm. \ref{thm:spiderboost})} & $\br{\frac{\sqrt{d}}{n\varepsilon}}^{1/2}$ &{\small \cite{wang2017differentially}}\\
      \cline{2-6}
      & {\small Population} & $\frac{1}{n^{1/3}} +\br{\frac{\sqrt{d}}{n\varepsilon}}^{1/2}$ \ \ & {\small (Thm. \ref{thm:accuracy_tree_private_Spider})} & $\sqrt{d}\varepsilon + \big(\frac{\sqrt{d}}{n\varepsilon} \big)^{1/2}$ &{\small \cite{ZCHSB20}}\\
      \hline
      \multirow{1}{*}{{\small Convex}} & {\small Population} & $\frac{1}{\sqrt{n}} + \frac{\sqrt{d}}{n\varepsilon}$\ \ &  {\small(Thm. \ref{thm:population-convex-gradient-small})} & {\small None} &\\
      \hline
      {\renewcommand{\arraystretch}{1}\multirow{2}{*}{\small{
      \begin{tabular}{c}{\small Non-convex} \\ {\small GLM}\end{tabular}}}}
        &  {\small Empirical} &
        $\big[\frac{\sqrt{\rank}}{n\varepsilon}\big]^{2/3} \!\land\!\frac{1}{(n\epsilon)^{2/5}}$
        & {\small (Cor. \ref{cor:glm-corollary})} & $\br{\frac{\sqrt{\rank}}{n\varepsilon}}^{1/2}$ & {\small \cite{song2021evading}}\\
    \cline{2-6}
      & {\small Population} & 
      {\small $\frac{1}{\sqrt{n}} + \big[\frac{\sqrt{\rank}}{n\varepsilon}\big]^{2/3}\! \wedge\!\frac{1}{(n\epsilon)^{2/5}}$}
      & {\small (Cor. \ref{cor:glm-corollary})}  & {\small None} &\\
      \hline
       \multirow{1}{*}{{\small Convex GLM}} & {\small Population} & 
       {\small $\frac{1}{\sqrt{n}} + \frac{\sqrt{\rank}}{n\varepsilon}\!\wedge\! \frac{1}{\sqrt{n\epsilon}}$} \ \ &  {\small(Cor. \ref{cor:glm-corollary})} & {\small None} &\\
       \hline
    \end{tabular}
    \vspace*{3pt}
    \caption{Results summary: We omit log factors and function-class parameters. The symbol $\wedge$ stands for minimum of the quantities. %
    } %
    \label{tab:results_summary}
\end{table}}

%% file: sections/prelim.tex
\section{Preliminaries}

 Let $\lossf:\bbR^d \times \cX \rightarrow \bbR$ denote a (loss) function taking as input, the model parameter $w$ and data point $x \in \cX$. We assume that the function $w\mapsto \lossf(w;x)$ is $\lip$-Lipschitz and $\smooth$-smooth. That is, for all $x \in \cX$ and $w_1,w_2 \in \mathbb{R}^d$, $\abs{f(w_1;x)-f(w_2;x)}\leq \lip\norm{w_1-w_2}$ and $\norm{\nabla f(w_1;x)-\nabla f(w_2;x)}\leq \smooth\norm{w_1-w_2}$.
 Given a dataset $S \in \cX^n$ of $n$ points, we define the empirical risk as $\lossF(w;S) = \frac{1}{n}\sum_{i=1}^n\lossf(w;x_i)$. Assuming that the data points are sampled i.i.d. from an unknown distribution $\cD$, the population risk, denoted as $\lossF(w;\cD)$ is defined as $\lossF(w;\cD) = \mathbb{E}_{x\sim \cD}\lossf(w;x)$. Furthermore, we define $\fnot = \eloss(0;S) - \min_{w\in\re^d}\bc{\eloss(w;S)}$ %
 when discussing the empirical case and similarly for the population loss when discussing  stationary points of the population loss. We use $w^*$ to denote the population risk minimizer. Finally, we use the notation $\mathbb{I}_d$ to denote the $d\times d$ identity matrix and use $[a]$ to denote the set  $\bc{1,2,...,a}$ for $a \geq 1$.

\paragraph{Stationary points:} Given a dataset $S$, our goal is to find an $\alpha$-stationary point $\out$ of either empirical or population risk; formally, $\norm{\nabla \lossF(\out;S)}\leq \alpha$ or  $\norm{\nabla \lossF(\out;\cD)}\leq \alpha$,   respectively.

\paragraph{Differential Privacy (DP)~\cite{dwork2006calibrating}:}
An algorithm $\cA$ is $(\varepsilon,\delta)$-differentially private if for all datasets $S$ and $S'$ differing in one data point and all events $\cE$ in the range of the $\cA$, we have, $\mathbb{P}\br{\cA(S)\in \cE} \leq     e^\varepsilon \mathbb{P}\br{\cA(S')\in \cE}  +\delta$.

\paragraph{Generalized Linear Models (GLMs):} For data domain $\cX\subseteq \bbR^d$ and $\cY\subseteq\bbR$, a loss function $f:\bbR^d \times \cX \times \cY \rightarrow \bbR$ is a GLM if $f(w;(x,y)) = \phi_y\br{\ip{w}{x}}$ for some function $\phi_y$. Our result for GLMs uses random matrices which satisfy the Johnson-Lindenstrauss (JL) property, defined as follows. 
\begin{definition}[$(\gamma,\beta)$-JL property] A random matrix $\Phi \in \bbR^{k\times d}$ %
satisfies $(\gamma,\beta)$-JL property if for any 
$u,v \in\bbR^d$, $\mathbb{P}\left[\abs{\ip{\Phi u}{\Phi v} - \ip{u}{v}} > \gamma\norm{u}\norm{v}\right]\leq \beta.$
\end{definition}

%% file: Arxiv/arxiv-emp-sp.tex
\section{Stationary Points of Empirical Risk} \label{sec:esp}

\subsection{Efficient Algorithm with Faster Rate} \label{sec:efficient-esp}
The algorithm for our upper bound is a noisy version of the SpiderBoost algorithm \cite{spiderboost}\footnote{SpiderBoost itself is essentially the Spider algorithm \cite{fang2018spider} with a different learning rate and analysis.}. %
The algorithm works by running a series of phases of length $q$. Each phase starts with a minibatch estimate of the gradient, and subsequent gradient estimates within the phase are then computed by adding an estimate of the gradient variation. The key to the analysis is to bound the error in the gradient estimate at each iteration.
Towards this end, we have the following generalization of the \cite[Lemma 1]{spiderboost}, which follows directly from \cite[Proposition 1]{fang2018spider}.

\begin{algorithm}[t]
\caption{\label{alg:spider} Private SpiderBoost}
\begin{algorithmic}[1]
\REQUIRE Dataset: $S\in\cX^n$, Function: $f:\re^d\times\cX\mapsto\re$, Learning Rate: $\eta$, Phase Size: $q$, Batch Sizes $b_1,b_2$, Privacy Parameters: $(\varepsilon,\delta)$, Iterations: $T$
\STATE $w_0 = 0$

\STATE $\sigma_1 = \frac{c\lip\sqrt{\log{1/\delta}}}{\varepsilon}\max\bc{\frac{1}{b_1}, \frac{\sqrt{T}}{\sqrt{q}n}}$, where $c$ is a universal constant.
\STATE $\sigma_2 = \frac{c\smooth\sqrt{\log{1/\delta}}}{\varepsilon}\max\bc{\frac{1}{b_2}, \frac{\sqrt{T}}{n}}$
\STATE $\hat{\sigma}_2 = \frac{2c\lip\sqrt{\log{1/\delta}}}{\varepsilon}\max\bc{\frac{1}{b_2}, \frac{\sqrt{T}}{n}}$

\FOR{$t=0, \ldots, T$}
    \IF{$\mod(t,q)=0$}
    \vspace*{2pt}
        \STATE Sample batch $S_t$ of size $b_1$
        \vspace*{2pt}
        \STATE Sample $g_t \sim \cN(0,\mathbb{I}_d \sigma_1^2)$
        \vspace*{2pt}
        \STATE $\grad_t = \frac{1}{b_1}\sum_{x\in S_t} \nabla f(w_t;x) + g_t$  \label{line:grad_oracle}
    \ELSE
        \STATE Sample batch $S_t$ of size $b_2$
        \STATE Sample $g_t \sim \cN\br{0,\mathbb{I}_d \min\bc{\sigma_2^2\norm{w_t - w_{t-1}}^2,\hat{\sigma}_2^2}}$
        
        \STATE $\Delta_t = \frac{1}{b_2}\sum_{x\in S_t} \left[ \nabla f(w_t;x) - \nabla f(w_{t-1};x) \right] + g_t$ \label{line:gv_oracle} 
        
        \vspace*{2pt}
        \STATE $\grad_t = \grad_{t-1} + \gv_t$ %
        \vspace*{1pt}
    \ENDIF
\STATE $w_{t+1} = w_t - \eta \grad_t$ 
\ENDFOR
\STATE return $\out$ uniformly at random from $w_1,\ldots,w_{T}$
\end{algorithmic}
\end{algorithm}

\begin{lemma} \label{lem:spider-grad-err}
Consider Algorithm \ref{alg:spider}, and for any $t\in\bc{0,..,T}$ let $s_t = \floor{\frac{t}{q}} q$.
If each $\grad_t$ computed in line \ref{line:grad_oracle} is an unbiased estimate of %
$\nabla\elossS{w_t}$ satisfying $\ex{}{\norm{\grad_{s_t} - \nabla \elossS{w_{s_t}}}^2} \leq \tau_1^2$ 
and each $\Delta_t$ computed in line \ref{line:gv_oracle} is an unbiased estimate of the gradient variation satisfying $\ex{}{\norm{\Delta_t - [\nabla \elossS{w_t} - \nabla \elossS{w_{t-1}}]}^2} \leq \tau_2^2\norm{w_t-w_{t-1}}^2$. 
Then for any $t \geq s_t + 1$, the
iterates of Algorithm \ref{alg:spider} satisfy
\begin{align*}
    \ex{}{\norm{\grad_t - \nabla \eloss(w_t)}^2} \leq \tau_2^2\sum_{k=s_t+1}^t \ex{}{\norm{w_k - w_{k-1}}^2} + \tau_1^2.
\end{align*}
\end{lemma}
For privacy, using smoothness we observe the sensitivity of the gradient variation estimate at iteration $t$ is proportional to $\beta\norm{w_t-w_{t-1}}$. Thus we can apply the above lemma with $\tau_1^2 = \frac{\lip^2}{b_1} + \lip^2\sigma_1^2$ and $\tau_2^2 = \frac{\smooth^2}{b_2} + \smooth^2\sigma_2^2$ (note the Gaussian noise in line \ref{line:gv_oracle} is drawn with variance scale at most  $\sigma_2^2\norm{w_t-w_{t-1}}^2$). By carefully balancing the algorithm parameters, we are then able to obtain the following result. The full proof is deferred to Appendix \ref{app:spiderboost}.

\begin{theorem}[\label{thm:spiderboost}Private Spiderboost ERM] 
Let $\varepsilon,\delta\in[0,1]$ and $n \geq \max\bc{\frac{(\lip\varepsilon)^2}{\fnot \smooth d\log{1/\delta}}, \frac{\sqrt{d}\max\bc{1,\sqrt{\smooth\fnot}/\lip}}{\varepsilon}}$.
Algorithm \ref{alg:spider} is $(\varepsilon,\delta)$-DP.
Further, there exist settings of $T,\eta,q,b_1,b_2$ such that Algorithm \ref{alg:spider} 
satisfies 
\begin{align*}
    \ex{}{\norm{\nabla F(\out;S)}} 
    &= O\br{\br{\frac{\sqrt{\fnot\smooth\lip}\sqrt{d \log{1/\delta}}}{n\varepsilon}}^{2/3} + \frac{\lip\sqrt{d\log{1/\delta}}}{n\varepsilon} }
\end{align*}
and has oracle complexity $\tilde{O}\br{\max\bc{\br{\frac{n^{5/3}\varepsilon^{2/3}}{d^{1/3}}}, \br{\frac{n\varepsilon}{\sqrt{d}}}^2 }}$. 
\end{theorem}

In the case where the dominant error term is $\alpha = \tilde{O}\Big(\big[\frac{\sqrt{d}}{n\varepsilon}\big]^{2/3}\Big)$, then we approximately have 
oracle complexity $\tilde{O}\big(\max\big\{\frac{1}{\alpha^3}, \frac{n}{\alpha}\big\}\big)$.

\subsection{Lower Bound}\label{sec:lowerbound}
We now show a lower bound for the sample complexity of finding a stationary point under differential privacy in the unconstrained setting, which shows that the $O\big(\frac{\lip\sqrt{d\log{1/\delta}}}{n\varepsilon}\big)$ term in the rate given in Theorem \ref{thm:spiderboost} is necessary. Furthermore, as our lower bound holds for all levels of smoothness, it also shows that our rate in Theorem \ref{thm:spiderboost} 
is optimal in the (admittedly uncommon) regime where 
$\smooth \leq \frac{\sqrt{d}\lip^2}{\fnot n\varepsilon}$.
Our lower bound in fact holds even for convex functions. Furthermore, this result implies the same lower bound (up to log factors) for the population gradient using the technique in \cite[Appendix C]{BFTT19}.

\begin{theorem}
\label{thm:lower-bound-convex}
Given $\lip,\smooth,n,\varepsilon  = O(1),2^{-\Omega(n)}\leq \delta\leq 1/n^{1+\Omega(1)}$, there exists an $\lip$-Lispchitz, $\smooth$-smooth (convex) loss $\lossf: \bbR^d \times \cX \rightarrow \bbR$ and a dataset $S$ of $n$ points 
such that any $(\varepsilon,\delta)$-DP algorithm run on $S$ with output $\out$ satisfies,
\begin{align*}
    \norm{\nabla \eloss(\out;S)} = \Omega\br{\lip\min\br{1,\frac{\sqrt{d \log{1/\delta}}}{n\varepsilon}}}.
\end{align*}
\end{theorem}

The proof 
is based on a reduction to DP mean estimation. Specifically, we consider a instance of the Huber loss
    function for which the minimizer is the empirical mean of the dataset. We then argue that close to the minimizer, the empirical stationarity is lower bounded by DP mean estimation bound \cite{SU15}, and far away, by construction, the empirical stationarity is $\lip$. We defer the details to Appendix \ref{app:lower-bounds}.

\paragraph{Challenges for Further Rate Improvements}
Given the above lower bound, the question arises as to whether the $\tilde{O}\big(\big[\frac{\sqrt{d}}{n\varepsilon}]^{2/3}\big)$ term can be improved. An informal argument using the oracle complexity lower bound of \cite{arjevani2019lower} suggests several major challenges in obtaining further rate improvements.
A more detailed version of the following discussion can be found in Appendix \ref{app:rate-challenges}.

Consider methods which ensure privacy by directly privatizing the gradient/gradient variation queries. 
The aim of such methods is to design some private stochastic first order oracle, $\cO_{\varepsilon',\delta'}$, such that a set of $G$ queries to $\cO_{\varepsilon',\delta'}$ satisfies $(\varepsilon,\delta)$-DP, and use this oracle in some optimization algorithm $\cA(\cO_{\varepsilon',\delta'})$. 
Such a setup encapsulates numerous results in the convex setting \cite{BFTT19,KLL21}, and is even more dominant in non-convex settings \cite{wang2017differentially,ZCHSB20,Abadi16}.
Under advanced composition based arguments, to make $G$ calls to such a private oracle one needs $\varepsilon' \leq \varepsilon/\sqrt{G}$. 
Now, standard fingerprinting code arguments suggest lower bounds on the level of accuracy of any such private oracle \cite{SU15}. Specifically, without leveraging further problem structure beyond Lipschitzness, one needs the gradient estimation error to be at least $\tau_1 = \Omega\Big(\frac{\lip\sqrt{Gd\log{1/\delta}}}{n\varepsilon}\Big)$. A similar argument suggests the error in the gradient variation between iterates $w,w'$ must at least $\tau_2\norm{w-w'} = \Omega\Big(\frac{\smooth\norm{w-w'}\sqrt{Gd\log{1/\delta}}}{n\varepsilon}\Big)$. Now consider some optimization algorithm, $\cA$, which takes as input a stochastic oracle $\cO$ for some smooth function $\cL$. The lower bound of \cite{arjevani2019lower} suggests that if $\cA$ makes at most $G$ queries to $\cO$, the algorithm satisfies 
$\ex{}{\norm{\nabla \cL(\cA(\cO))}} = \Omega\br{\br{\frac{\fnot \tau_2\tau_1}{G}}^{1/3} + \frac{\tau_1}{\sqrt{G}}}$. If $\cO$ is a private oracle satisfying the previously mentioned conditions, we would then have under the setting of $\tau_1$ and $\tau_2$ suggested by privacy that 
$$\ex{}{\norm{\nabla \cL(\cA(\cO))}} = \Omega\br{\br{\frac{\sqrt{\fnot\smooth\lip}\sqrt{d \log{1/\delta}}}{n\varepsilon}}^{2/3} + \frac{\lip\sqrt{d\log{1/\delta}}}{n\varepsilon}}.$$ 

This indicates a substantial challenge for future rate improvements, as 
alternative methods which avoid private gradients (see e.g. \cite{feldman2020private}) rely crucially on stability guarantees arising from convexity.

%% file: Arxiv/arxiv-pop-sp-v2.tex
\section{Stationary Points of Population Risk}

\begin{algorithm}[H]
\caption{Tree-based Private Spider}
\begin{algorithmic}[1]
\REQUIRE $S=(x_1,\ldots,x_n)\in{\cal X}^n$: private dataset, $(\varepsilon, \delta)$: privacy parameters, $T$: number of rounds, $b$: batch size at beginning of each round, $D$: depth of trees at each round, $\beta$: step-size parameter, $\tilde{\alpha}$: accuracy parameter.
\STATE $w_{0, \ell(2^{D}-1)} = 0$
\FOR{$t =1$ to $T$}
\STATE Set $w_{t, \varnothing}=w_{t-1, \ell(2^{D}-1)}$
\STATE Draw a batch $S_{t, \varnothing}$ of $b$ data points, %
set $S\leftarrow S\setminus S_{t,\varnothing}$.
\STATE Set $\sigma_{t, \varnothing}^{2}:=\frac{8 L_{0}^{2} \log{1.25 / \delta}}{b^{2} \varepsilon^{2}}$.
\STATE \label{alg:SPIDER_first_grad} $\nabla_{t, \varnothing}=\frac{1}{b} \sum_{x \in S_{t, \varnothing}} \nabla f\left(w_{t, \varnothing}; x\right)+g_{t, \varnothing}$, where $g_{t, \varnothing} \sim \mathcal{N}\left(0, \mathbb{I}_d \sigma_{t, \varnothing}^{2}\right)$.
\FOR{$u_{t, s}\in\operatorname{DFS}\left[D\right]$} 
\STATE Let $s=\hat{s} c$, where $c \in\{0,1\}$.  
\IF{$c=0$} 
\STATE $\nabla_{t, s}=\nabla_{t, \hat{s}}$
\STATE $w_{t, s}=w_{t, \hat{s}}$
\ELSE 
\STATE Draw a batch $S_{t, s}$ of $\frac{b}{2^{|s|}}$ data points,
set $S\leftarrow S\setminus S_{t,s}$. 
\STATE Set noise variance $\sigma_{t, s}^{2}:=\frac{8\cdot 2^D\beta^2\log{1.25 / \delta}}{b^{2} \varepsilon^{2}}$.
\STATE %
$\Delta_{t, s} = \frac{2^{|s|}}{b}\!\!\sum\limits_{x \in S_{t, s}}\left(\nabla f\left(w_{t, s}; x\right)\!-\!\nabla f\left(w_{t, \hat{s}}; x\right)\right) + g_{t, s}$, where $g_{t, s} \!\!\sim \!\mathcal{N}\left(0, \mathbb{I}_d\sigma_{t, s}^{2}\right) .$ \label{alg:SPIDER_rec_grad}
\STATE $\nabla_{t, s}=\nabla_{t, \hat{s}}+\Delta_{t, s} .$
\ENDIF
\IF{$|s|=D$ (i.e, $u_{t,s}$ is a leaf)}
\IF{$\|\nabla_{t, s}\| \leq 2 \tilde{\alpha}$}
\STATE Return $w_{t, s}$ \label{alg:SPIDER_return}
\ENDIF
\STATE Let $u_{t,s^+}$ %
be the next vertex in $\mathrm{DFS}[D]$.
\STATE \label{alg:SPIDER_step} Set $\eta_{t, s}:=\frac{\beta}{2^{D/2}L_1\|\nabla_{t, s}\|}$
\STATE $w_{t, s^+} = w_{t, s} - \eta_{t, s} \nabla_{t, s} .$
\ENDIF
\ENDFOR 
\ENDFOR
\STATE Return $\overline{w}$, chosen uniformly at random from $\{w_{t,s} : t\in[T], %
u_{t,s} \text{ is a leaf}\}$.\label{alg:tree_final_line}
\end{algorithmic}
\label{Alg:tree_private_spider}
\end{algorithm}

For the population gradient, we provide a linear time algorithm; see Algorithm \ref{Alg:tree_private_spider} for pseudocode. 
It is a noisy variant of SPIDER \cite{fang2018spider}, and utilizes a variance reduction technique tailored to an underlying binary tree structure. 
Namely, we run $T$ rounds, where at the beginning of round $t$ we build a binary tree of depth $D$, whose nodes are denoted by $u_{t,s}$, where $s \in \{0,1\}^D$. Every node $u_{t,s}$ is associated with a parameter vector $w_{t,s}$ and a gradient estimate $\nabla_{t,s}$%
. Next, we perform a Depth-First-Search traversal of the tree. We denote by DFS$[D]$ the set of nodes in the visiting order  %
excluding the root, for example: DFS$[2]=\{u_0,u_{00},u_{01},u_1,u_{10},u_{11}\}$. %
When a left child node is visited, it receives the same parameter vector and gradient estimator %
of the parent node. 

On the other hand, when a right child node is visited, it receives a fresh set of samples and uses it to update the gradient estimator coming from the parent node. Every time a leaf node is reached, a gradient step is performed using the gradient estimator associated to the leaf. Finally, the parameter vector of a right child node comes from the gradient step performed at the right-most leaf in the left sub-tree of it. The use of the binary tree structure is benefitial because every gradient estimator is updated at most $D$ times within a round of $2^D$ optimization steps, as opposed to the original SPIDER algorithm where the gradient estimators are updated at every optimization step. This way, we are able to perform the same number of optimization steps but adding substantially smaller amounts of noise, leading to a faster rate than the one we would get without using the tree. In the following, we denote by $\ell(k)$ the binary representation of any number $k\in [0, 2^D-1]$ and by $|s|$ the depth of $u_{t,s}$ for any $t \in [T]$.

The proposed algorithm is similar to the one in Section $5$ of \cite{bassily2021non} for constrained Differentially Private-Stochastic Convex Optimization (DP-SCO), with the key difference that Algorithm \ref{Alg:tree_private_spider} executes each round with fixed depth trees, 
which is key for our convergence analysis, whereas the prior work leverages convexity to construct trees that increase depth by one at each round. In addition, to choose the step-size in \cite{bassily2021non} the authors leverage the bounded diameter of the domain, 
while our step-size is chosen as that of \cite{fang2018spider}, i.e.~normalized by the norm of the gradient estimator and proportional to the target accuracy. This choice is crucial for controlling the sensitivity of the gradient variation estimator in the unconstrained setting, and consequently for the privacy analysis as well. Our results are presented below and the proofs are deferred to Appendix \ref{sec:app_pop}. 

\begin{theorem}[Privacy guarantee]\label{thm:privacy_tree_private_spider} For any $\varepsilon, \delta \in [0,1]$, Algorithm \ref{Alg:tree_private_spider} is $(\varepsilon, \delta)$-DP.
\end{theorem}
\begin{theorem}[Accuracy guarantee]\label{thm:accuracy_tree_private_Spider}

Let $p\in (0,1)$, $\varepsilon, \delta > 0$, $b = \max\bc{n^{2/3}, \frac{\sqrt{n}d^{1/4}}{\sqrt{\varepsilon}}}$, $D$ be such that $D2^{D+1} = b$, $T = \frac{n}{b(D/2+1)}$, $\alpha =  \sqrt{2}L_0\max\big\{\frac{1}{n^{1/3}}, \big(\frac{\sqrt{d}}{n\varepsilon}\big)^{1/2}\big\}$, $\beta = \alpha\min\{1,\frac{\sqrt{b}\varepsilon}{\sqrt{d}}\}$, and $\tilde{\alpha} = \tilde{C}\alpha$, where $\tilde{C} = 256\log{\frac{1.25}{\delta}}\log{\frac{2T2^{D+1}}{p}} + \frac{8L_1F_0\sqrt{2D}(D/2+1)}{2L_0^2}$. Then, for any %
$n \ge \max\{\sqrt{d}(\frac{D}{2}+1)^2/\varepsilon, (\frac{D}{2}+1)^3\}$, with probability $1-p$, Algorithm \ref{Alg:tree_private_spider} ends in line \ref{alg:SPIDER_return}, returning an iterate $w_{t,s}$ with 
\[\|\nabla F(w_{t,s};{\cal D})\| \leq 3\sqrt{2}L_0\tilde{C}\max\Big\{\frac{1}{n^{1/3}}, \Big(\frac{\sqrt{d}}{n\varepsilon}\Big)^{1/2}\Big\}.\]
    
Furthermore, Algorithm \ref{Alg:tree_private_spider} has oracle complexity of $n$.
\end{theorem}

%% file: Arxiv/arxiv-convex.tex
\begin{algorithm}[h]
\caption{Recursive Regularization}
\label{alg:noisyRR}
\begin{algorithmic}[1]
\REQUIRE Dataset $S$, 
loss function $\lossf$, steps $T$, $\bc{\lambda_t}_t$, $\bc{R_t}_t$,
$\mathsf{PrivateSubRoutine}$,
number of steps of sub-routine $\bc{K_t}$,
selector functions $\bc{\cS_t(\cdot)}_t$, 
step size $\bc{\eta_t}_t$,
noise variances $\bc{\sigma_t}_t$
\STATE $w_0 = 0$, $n_0=1$
\STATE Define function $(w,x)\mapsto \lossf
^{(0)}(w;x) = \lossf(w;x) + \frac{\lambda_0}{2}\norm{w-w_0}^2$
\FOR{$t=1$ to $T-1$}
 \STATE $n_t = n_{t-1}+
 \left\lfloor\frac{\abs{S}}{T}\right\rfloor$
\STATE $\out_t = \mathsf{PrivateSubRoutine}\br{S_{n_{t-1}:n_t},\lossf^{(t-1)},R_t,K_t,\eta_t,\cS_t(\cdot),\sigma_t}$
\STATE Define function $(w,x)\mapsto \lossf^{(t)}(w;x) = \lossf^{(t-1)}(w;x) + \frac{\lambda_t}{2}\norm{w-\out_t}^2$
\ENDFOR
\ENSURE{$\out = \out_{T}$}
\end{algorithmic}
\end{algorithm}

In this section, we additionally assume that the loss function is convex. The motivation for this is two-fold: firstly, this setting has recently gained attention in a non-private setting~\cite{nesterov2012make,allen2018make,foster2019complexity}. Secondly, in this setting we are able to establish tightly the sample complexity of approximate stationary points.

Our method is based on the recursive regularization technique proposed in \cite{allen2018make}, and further improved by \cite{foster2019complexity}. The main idea, as the name suggests, is to recursively regularize the objective and  optimize it via some solver. For the DP setting, the key idea is to use 
a private sub-routine as 
the inner solver. Furthermore, while a solver for the unconstrained problem suffices non-privately,
we need to carefully increase the radius of the constrained set over which the solver operates.

\begin{theorem}
\label{thm:population-convex-gradient-small}
Let $\lip,\smooth, \varepsilon, \delta >0$, $d, n \in \bbN$.
Let $w\mapsto\lossf(w;x)$ be an $\lip$-Lipschitz $\smooth $-smooth convex function for all $x$. 
Let $R_t = \br{\sqrt{2}}^{t}\norm{w^*}, \lambda_t = 2^t \lambda$,
$\eta_t = \frac{\log{K_t}}{\lambda_tK_t}$,
$T = \left\lfloor \mathsf{log}_2\br{\frac{\smooth}{\lambda}} \right\rfloor$,
 $\sigma_t^2 = \frac{64\lip^2K_t^2\log{1/ \delta}}{n^2\varepsilon^2}$, and
$\cS_t(\bc{w_k}_k) = \frac{1}{\sum_{k=1}^{K_t}\br{1-\eta_t \lambda_t}^{-k}}\sum_{k=1}^{K_t}\br{1-\eta_t \lambda_t}^{-k}w_k$.

\begin{CompactEnumerate}
\item (Optimal rate) Algorithm \ref{alg:noisyRR} run with $\mathsf{NoisyGD}$ (Algorithm \ref{alg:noisyGD} in 
Appendix~\ref{sec:app_convex}) as the PrivateSubRoutine with above parameter settings and 
 $\lambda = \frac{\lip^2}{\smooth \norm{w^*}}\min\br{\frac{1}{n},\frac{d}{n^2\varepsilon^2}}$ and
 \linebreak $K_t = \max\br{\frac{\smooth+\lambda_t}{\lambda_t}\log{\frac{\smooth+\lambda_t}{\lambda_t}},\frac{n^2\varepsilon^2 \br{\lip^2\lambda+\smooth^{3/2}}}{T^2\lambda d\lip^2\log{1/\delta}}}$
 satisfies
$(\varepsilon,\delta)$-DP, and given a dataset $S$ of $n$ i.i.d. samples from $\cD$, outputs $\out$ such that
 \begin{align*}
    \mathbb{E}\norm{\nabla \plossD{\out}} =
    \tilde O\br{\frac{\lip}{\sqrt{n}}+\frac{\lip\sqrt{d}}{n\varepsilon}}.
\end{align*}
Furthermore, the above rate is tight up to poly-logarithmic factors.
 \item (Linear time rate) Algorithm \ref{alg:noisyRR} run with $\mathsf{PhasedSGD}$ (Algorithm \ref{alg:phased-sgd}) as the PrivateSubRoutine with with above parameter settings and 
 $\lambda = \max\br{\frac{\lip^2}{\smooth \norm{w^*}^2}\min\br{\frac{1}{n},\frac{d}{n^2\varepsilon^2}}, \frac{\smooth\log{n}}{n}}$ and
 $K_t = \lfloor\frac{n}{T}\rfloor$
satisfies
$(\varepsilon,\delta)$-DP and given a dataset $S$ of $n$ i.i.d. samples from $\cD$, in linear time, outputs $\out$ with
\begin{align*}
    \mathbb{E}\norm{\nabla \plossD{\out}} =
    \tilde O\br{\frac{\lip}{\sqrt{n}}+\frac{\lip\sqrt{d}}{n\varepsilon} + \frac{\smooth\norm{w^*}}{\sqrt{n}}}.
\end{align*}
\end{CompactEnumerate}
\end{theorem}

The proof of the above result is deferred to Appendix \ref{sec:app_convex}.
For the tightness of the rate, the necessity of the second term $\frac{\lip\sqrt{d}}{n\varepsilon}$ is
due to our DP empirical stationarity lower bound, Theorem \ref{thm:lower-bound-convex}.
For the first ``non-private'' term $\frac{L_0}{\sqrt{n}}$,
even though \cite{foster2019complexity} proved a sample complexity
lower bound, their instance is not Lipschitz and has $d=\Omega\br{n\log{n}}$, hence not applicable.
To remedy this, we give a new lower bound construction with a Lispchitz function in $d=1$, Theorem \ref{thm:sample-complexity-lower-bound} in Appendix \ref{app:lower-bounds}. 
The polylog dependence on $\smooth$ and $\norm{w^*}$ in the upper bounds, is consistent with the non-private sample complexity in \cite{foster2019complexity}.

The second result is a linear time method which has an additional $\smooth \norm{w^*}/\sqrt{n}$ term. 
Firstly, if the smoothness parameter is \textit{small enough}, then there is no overhead; this small-enough smoothness is precisely the regime in which we have linear time methods with optimal rates for smooth DP-SCO  \cite{feldman2020private}.
More importantly, 
\cite{foster2019complexity} showed that even in the non-private setting,
a polynomial dependence on $\smooth \norm{w^*}$ is necessary in the stochastic oracle model.
However, the optimal non-private term, shown in \cite{foster2019complexity}, is $\smooth \norm{w^*}/n^2$, achieved by accelerated methods.
Improving this dependency, if possible, is an interesting direction for future work.

%% file: sections/glm.tex
\section{Generalized Linear Models}

In this section, we assume that the loss function is a generalized linear model (GLM), $\lossf(w;(x,y))=\phi_y\br{\ip{w}{x}}$. Also, assume the norm of data points $x$ are bounded by $\norm{\cX}$ and the function $\phi_y: \bbR \rightarrow \bbR$ is $\lip$-Lipschitz and $\smooth$-smooth for all $y$. Furthermore, let $\rank$ denote the rank of design matrix $X\in \bbR^{n\times d}$.

\begin{algorithm}[h]
\caption{JL method}
\label{alg:jlmethod}
\begin{algorithmic}[1]
\REQUIRE Dataset $S$, 
 function $(z,y)\mapsto \phi_y(z)$, Algorithm $\cA$, JL matrix $\Phi \in \bbR^{k\times d}$, $\lip$, $\smooth$, $\norm{\cX}$
\STATE $\tilde w = \cA((z,y)\mapsto \phi_y(z), \bc{(\Phi x_i,y_i)}_{i=1}^n, 2\lip\norm{\cX}, 2\smooth\norm{\cX}^2, \varepsilon,\delta/2)$
\ENSURE{$\out = \Phi^\top \tilde w$}
\end{algorithmic}
\end{algorithm}

Algorithm \ref{alg:jlmethod} is a generic method which converts \textit{any}
for smooth Lipschitz losses with an empirical stationarity guarantee to get dimension-independent rates on population stationarity for smooth Lipschitz GLMs.
This algorithm is the JL method from \cite{arora2022differentially} used therein to give excess risk bounds for convex GLM. We note that while the JL method there is limited to the Noisy GD method, ours is a black-box reduction.
Furthermore, unlike \cite{arora2022differentially}, we show that the JL method gives finer rank based guarantees by leveraging the fact it acts as an oblivious approximate subspace embedding (see Definition \ref{defn:ose} in Appendix \ref{app:glm}).

\begin{theorem}
\label{thm:jl-reduction}

Let $\cA$ be an $\br{\varepsilon,\delta}$-DP algorithm which when run on a $\smooth$-smooth $\lip$-Lipschitz function on a dataset $S = \bc{(x_i,y_i)}_{i=1}^n$ where $x_i \in \cX\subseteq \bbR^d$, guarantees $\mathbb{E}\left[\norm{\nabla  \lossF(\cA(S);S)}\right]\leq g(d,n, \smooth, \lip, \varepsilon,\delta)$
and $\norm{\cA(S)} \leq \text{poly}(n,d, \lip,\smooth)$ with probability
at least $1-\frac{1}{\sqrt{n}}$. Then,  Algorithm \ref{alg:jlmethod} run with 
$$k = \left\lceil \min\br{\argmin_{j \in \bbN}\br{g(j,n,2\lip\norm{\cX},2\smooth\norm{\cX}^2,\varepsilon,\delta/2) + \frac{\lip\norm{\cX}\log{n}}{\sqrt{j}}}, \mathsf{rank} \log{\frac{2n}{\delta}}} \right\rceil$$
on a $\lip$-Lipschitz, $\smooth$-smooth GLM loss, is $\br{\varepsilon,\delta}$-DP. Furthermore, given a dataset of $n$ i.i.d samples from $\cD$,  its output $\out$ satisfies,
\begin{align*}
    \mathbb{E}\left[\norm{\nabla \lossF(\out;\cD)}\right] \leq \tilde O\br{\frac{\lip\norm{\cX}}{\sqrt{n}} + g(k,n,2\lip\norm{\cX},2\smooth\norm{\cX}^2,\varepsilon,\delta/2)}
\end{align*}
\end{theorem}

The expression for $k$ above comes from the subspace embedding property of JL, and from balancing the dimension of the embedding with respect to the error of $\cA$ and the approximation error of the JL embedding.
The proof is based on the properties of JL matrices: oblivious subspace embedding and preservation of norms, together with a new uniform convergence result for gradients of Lipschitz GLMs.
The full proof is deferred to Appendix \ref{app:glm}.

Below, we instantiate the above with our proposed algorithms.

\begin{corollary}
\label{cor:glm-corollary}
Under the assumptions of Theorem \ref{thm:jl-reduction},
Algorithm \ref{alg:jlmethod} run with $\cA$ as
\begin{CompactEnumerate}
    \item Private $\!\!\!$ Spiderboost (Alg.~\ref{alg:spider}) yields $\norm{\nabla \lossF(\out;\cD)}\!=\!\tilde O\br{\frac{1}{\sqrt{n}}+\min\br{\br{\frac{\sqrt{\rank}}{n\varepsilon}}^{2/3},\frac{1}{\br{n\varepsilon}^{2/5}}}}.$
    \item Algorithm \ref{alg:noisyRR} with NoisyGD as PrivateSubRoutine, under the additional assumption that $w\mapsto \lossf(w;(x,y))$ is convex for all $x,y$, yields $\norm{\nabla \lossF(\out;\cD)}=\tilde  O\br{\frac{1}{\sqrt{n}} + \min\br{\frac{\sqrt{\rank}}{n\varepsilon},\frac{1}{\sqrt{n\varepsilon}}}}$.
\end{CompactEnumerate}
\end{corollary}

We remark that the above technique also gives bounds on empirical stationarity. In particular, the first term $\frac{1}{\sqrt{n}}$, in the above guarantees, is the uniform convergence bound and the second term is the bound on empirical stationarity.

%% file: appendix/app-lower-bounds.tex
\section{Lower bounds}
\label{app:lower-bounds}

\subsection{Missing details from DP Empirical Stationarity Lower Bound}
\begin{proof}[Proof of Theorem \ref{thm:lower-bound-convex}]
For any $r>0$, let $\cW_r$ denote the ball of radius $r$ centered at the origin.
Let $B = \frac{\lip}{\smooth}$. 
Consider the loss function: %
\begin{align*}
    \lossf(w;x) = \begin{cases}
    \frac{\smooth}{2}\norm{w-x}^2 & \text{if} \norm{w-x} \leq B\\
    \lip\norm{w-x} - \frac{\lip^2}{2\smooth} & \text{ otherwise}
    \end{cases}
\end{align*}

The function $\lossf(w;x)$ is convex, $\smooth$-smooth and $\lip$-Lispchitz in $\bbR^d$. 
We restrict to datasets $S =\bc{x_i}_{i=1}^n$ where $x_i \in \cW_{B/4}$ for all $i$,
and let $\lossF(w;S) = \frac{1}{n}\sum_{i=1}^n\lossf(w;x_i)$ be the empirical risk on $S$.
The unconstrained minimizer of $\lossF(w;S)$ is $w^* = \frac{1}{n}\sum_{i=1}^n x_i$ which lies in $\cW_{B/4}$.

For any $w \in \cW_{3B/4}$, $w$ lies in the quadratic region around all data points. 
Hence, from $\smooth$-strong convexity  
of $w\mapsto \lossF(w;S)$ on $\cW_{3B/4}$,  we have that whenever $\out \in \cW_{3B/4}$,
\begin{align*}
   \norm{\nabla \lossF(\out;S)}\norm{\out - w^*} \geq \ip{\nabla \lossF(\out;S)}{w^* - \out} \geq \lossF(\out;S) - \lossF(w^*;S) \geq \frac{\smooth}{2}\norm{\out - w^*}^2.
\end{align*}
Let $E$ be the event that $\bar w \in \cW_{3B/4}$ and let $\mathbb{E}_E$ denote the conditional expectation (conditioned on event $E$) operator.
Then,
\begin{align*}
   \mathbb{E}_{E}\norm{\nabla \eloss(\out;S)} \geq \frac{\smooth}{2} \mathbb{E}\norm{\out-w^*}
   \geq \frac{\smooth}{2} \Omega\br{\br{\frac{\lip}{4\smooth}}\min\br{1, \frac{\sqrt{d\log{1/\delta}}}{n\varepsilon}}}.
\end{align*}
where the last inequality follows from known lower bounds for DP mean estimation \cite{SU15, kamath2020primer}.
We remark that the lower bound in the referenced work is for algorithms which produce outputs in the ball of the same radius as the dataset, i.e. $\cW_{B/4}$. However, a simple post-processing  argument shows that the same lower bound applies to algorithms which produce output in $\cW_{3B/4}$. Specifically, assuming the contrary, we simply 
project the output in $\cW_{3B/4}$ to $\cW_{B/4}$: privacy is preserved by post-processing and the distance to the mean cannot increase by the non-expansiveness property of projection to convex sets, hence a contradiction. 
This gives us, 
\begin{align*}
    \mathbb{E}_E\left[\norm{\nabla \eloss(\out;S)}\right] \geq \Omega\br{\lip\min\br{1, \frac{\sqrt{d\log{1/\delta}}}{n\varepsilon}}}
\end{align*}

Let $\tilde \cW = \bc{w:\norm{w-w^*}\leq B/2}$. Since $\tilde \cW\subseteq \cW_{3B/4}$, we have that the above conditional lower bound applies for $\out \in \tilde \cW$ as well.
We now consider $\out \not \in \tilde \cW$.
Let $w'$ be \textit{any} point on the boundary of $\tilde \cW$, denoted as $\partial \cW$.
Note that $w'$ lies in the region where, for any data point, the corresponding loss is a quadratic function.
Hence, by direct computation, $\nabla F(w';S )  = \smooth \br{w'-w^*}$.
Therefore, 
\begin{align*}
    \ip{\nabla F(w')}{w'-w^*} = \smooth\norm{w'-w^*}^2 = \frac{\smooth B^2}{4}.
\end{align*}
We now apply Lemma \ref{lem:gradient-monotonicity} which gives us,
\begin{align*}
    \mathbb{E}_{E^c}\norm{\nabla F(\out;S)} \geq  \frac{\smooth B^2}{4} \cdot \frac{2}{B} =  \frac{\lip}{2},
\end{align*}
where $E^c$ denotes the complement set of $E$. 
We combine the above bounds using the law of total expectation as follows,
\begin{eqnarray*}
\mathbb{E}[\|\nabla F(\bar w;S)\|] %
&=& \mathbb{E}_E[\|\nabla F(\bar w;S)\|]\mathbb{P}\{\bar w\in E\}+\mathbb{E}_{E^c}[\|\nabla F(\bar w;S)\|]\mathbb{P}\{\bar w\in E^c\} \\
&=& \Omega\Big(L_0\min\Big\{1,\frac{\sqrt{d\log{1/\delta}}}{n\varepsilon}\Big\}\Big)\mathbb{P}(\bar w\in E) + \Omega(L_0)\mathbb{P}(\bar w\in E^c)\\
&=& \Omega \Big(L_0 \min\Big\{1,\frac{\sqrt{d\log{1/\delta}}}{n\varepsilon}\Big\}\Big).
\end{eqnarray*}

This completes the proof.
\end{proof}

\begin{lemma}
\label{lem:gradient-monotonicity}
Let $G,R\geq 0, d\in \bbN$. Let $\cW_R(w_0)$ denote the Euclidean ball around $w_0$ of radius $R$ and let $\partial \cW_R(w_0)$ %
denote its boundary.
Let $f:\bbR^d \rightarrow \bbR$ be a differentiable convex function. 
Suppose $w_0\in \bbR^d$ is such that
for every $v \in \partial \cW_R(w_0)$,
$\ip{\nabla f(v)}{v-w_0}\geq G$, then for any $w \not \in \cW_R(w_0)$, we have $\norm{\nabla f(w)}\geq \frac{G}{R}$.
\end{lemma}
\begin{proof}
For a unit vector $u\in \bbR^d$, define directional directive $f'_u(w) = \ip{\nabla f(w)}{u}$. We first show that for any $u\in \bbR^d:\norm{u}=1$ and any $w' \in \bbR^d$,  the function $f'_u(w' + ru)$ is non-decreasing in $r \in \bbR_+$. This simply follows from monotonicity of gradients since $f$ is convex. In particular, for any $r'>r>0$, we have
\begin{align*}
    f'_u(w' + r'u) - f'_u(w' + ru) &= \ip{\nabla f(w'+r'u) - \nabla f(w' + ru)}{u} \\
    & = \frac{1}{r'-r}\ip{\nabla f(w'+r'u) - \nabla f(w' + ru)}{w' + ru  - (w'+ru)}\\
    &>0
\end{align*}
We now prove the claim in the lemma statement. Let $w \not \in \partial W_R$ and define $u = \frac{w-w_0}{\norm{w-w_0}}$. Then from Cauchy-Schwarz inequality and the above monotonicity property, we have,
\begin{align*}
    \norm{\nabla f(w)} &\geq \ip{\nabla f(w)}{u} = f'_u(w) \geq f'_u(w_0 + Ru) = \ip{\nabla f(w_0 + Ru)}{u} \\
    & = \frac{1}{R}\ip{\nabla f(v)}{v-w_0} \geq \frac{G}{R}
\end{align*}
which finishes the proof.
\end{proof}

\subsection{Non-private Sample Complexity Lower Bound}

\begin{theorem}
\label{thm:sample-complexity-lower-bound}
For any $\lip,{\smooth}$, $n,d\in \bbN$, there exists a distribution $\cD$ over some set $\cX$ and a $\lip$-Lipschitz, ${\smooth}$-smooth (convex) loss function $ w \mapsto \lossf(w;x)$ such that given $n$ i.i.d samples from $\cD$, the output $\out$ of any algorithm satisfies,
\begin{align*}
    \mathbb{E}\norm{\nabla {\lossF}(\out;\cD)} = \Omega\br{\frac{\lip}{\sqrt{n}}}
\end{align*}
\end{theorem}

\begin{proof}

We construct a hard instance in $d=1$ dimension.
Let $p \in [0,1]$ 
be a parameter to be set later
and let $v\in\bc{-1,1}$ be %
chosen by an adversary. 
Let the data domain $\cX=\bc{-1,1}$ and consider the distribution $\cD$ on $\cX$ as follows:
\begin{align*}
    x = \begin{cases}
    1 &\text{with probability} \ \ \frac{1+vp}{2} \\
    -1 &\text{with probability} \ \  \frac{1-vp}{2} \\
    \end{cases}
\end{align*}
Note that $\mathbb{E}[x] = vp$.
Consider the loss function $\lossf(w;x)$ as
\begin{align*}
    \lossf(w;x) = \frac{\lip}{2}wx + \frac{\smooth}{2} \Delta(w)
\end{align*}
where $\Delta$ is the Huber regularization function, defined as,
\begin{align*}
\Delta(w) = \begin{cases}
\abs{w}^2 & \text{if} \ \abs{w}\leq \frac{\lip}{2\smooth}\\
\frac{\lip\abs{w}}{\smooth} -  \frac{\lip^2}{4\smooth^2}& \text{otherwise}
\end{cases}
\end{align*}
Note that the loss function $w\mapsto \lossf(w;x)$ is convex, $\lip$-Lipschitz and $\smooth$-smooth in $\bbR^d$, for all $x$.
The population risk function is, %
\begin{align*}
     \lossF(w;\cD) = \frac{\lip}{2}wpv + \frac{\smooth}{2}\Delta(w)
\end{align*}

Let $\out$ be output some algorithm given $n$ i.i.d. samples from $\cD$. Consider two cases:
\paragraph{Case 1: $\abs{\out}> \frac{\lip}{2\smooth}$:} The gradient norm in this case is
\begin{align*}
    \abs{\nabla \lossF(\out;\cD)}^2 &= \abs{\frac{\lip}{2}vp + \frac{\lip\out}{2\abs{\out}}}^2 \\
    & = \frac{\lip^2p^2}{4} + \frac{\lip^2}{4} + \frac{\lip^2}{2 \abs{\out}}vp{\out}\\
    & \geq \frac{\lip^2}{4} - \frac{\lip^2}{2}p\\
    &  = \frac{\lip^2}{4} - \frac{\lip^2}{8\sqrt{n}}\\
    & \geq \frac{\lip^2}{8}
\end{align*}
where
the first inequality follows since $v\frac{\out}{\abs{\out}}\geq -1$,
the third equality follows by setting $p=\frac{1}{\sqrt{16n}}$ and the second inequality follows since $n\geq 1$. %
We therefore have that $\mathbb{E}\abs{\nabla \lossF(\out;\cD)}\geq \frac{\lip}{2\sqrt{2}}$.

\paragraph{Case 2: $\abs{\out}\leq \frac{\lip}{2\smooth}$:} In this case, the gradient norm is,
\begin{align*}
    \abs{\nabla \lossF(\out;\cD)}^2 = \abs{\frac{\lip}{2}vp +\smooth\out}^2
\end{align*}
Suppose there exists 
an algorithm with output $\out$, which, with $n$ samples guarantees that $\mathbb{E}\abs{\nabla {\lossF}(\out;\cD)} < o\br{\frac{\lip}{\sqrt{n}}}$. Then from Markov's inequality, with probability at least $0.9$, we have that $\abs{\nabla {\lossF}(\out;\cD)}^2 < o\br{\frac{\lip^2}{n}}$. Let 
$\tilde w = -\frac{2\smooth\out}{\lip}$
, then we have that
with probability at least $0.9$,
\begin{align*}
    \abs{\nabla {\lossF}(\out;\cD)}^2 \leq o\br{\frac{\lip^2}{n}} \iff \abs{vp - \tilde w }^2 < o\br{\frac{1}{n}}
\end{align*}

This contradicts the well-known bias estimation lower bounds, with $p=\frac{1}{\sqrt{16n}}$, using Le Cam's method (\cite{duchi2016lecture}, Example 7.7), hence 
$\mathbb{E}\abs{\nabla {\lossF}(\out;\cD)}\geq \Omega\br{\frac{\lip}{\sqrt{n}}}$.
Combining the two cases finishes the proof.
\remove{
\enayat{old below}
Suppose there exists an algorithm the output $\out$ of which, with $n$ samples guarantees that $\mathbb{E}\norm{\nabla {\lossF}(\out;\cD)}^2 < \frac{\lip^2}{n}$, then define $\tilde w =- \frac{2{\smooth}\out}{\lip}$, we have,
\begin{align*}
    \mathbb{E}\norm{\nabla {\lossF}(\out;\cD)}^2 \leq \frac{\lip^2}{n} \iff \mathbb{E}\abs{vp - \tilde w }^2 < \frac{4}{n}
\end{align*}
This contradicts the well-known bias estimation lower bounds using Le Cam's method (\cite{duchi2016lecture}, Example 7.7), hence $\mathbb{E}\abs{\nabla {\lossF}(\out;\cD)}\geq \frac{\lip}{\sqrt{n}}$. Combining the two cases finishes the proof.
\enayat{OLD BELOW}
\cnote{Can we remove this already?}
\enayat{removed}
We construct a hard instance in $d=1$.
Let $p$ be a parameter to be set later
and let $v\sim \text{Unif}\br{\bc{-1,1}}$.
Let the data space $\cX=\bc{-1,1}$ and consider the distribution $\cD$ on $\cX$ as follows:
\begin{align*}
    x = \begin{cases}
    1 &\text{with probability} \ \ \frac{1+vp}{2} \\
    -1 &\text{with probability} \ \  \frac{1-vp}{2} \\
    \end{cases}
\end{align*}
Note that $\mathbb{E}[x] = pv$.
Consider the one dimensional GLM loss:
\begin{align*}
    \lossf(w;x) = \frac{\lip}{2}wx + \frac{{\smooth}}{2}\br{ \mathbbm{1}_{\abs{wx}\leq \frac{\lip}{2{\smooth}}} w^2x^2 + \mathbbm{1}_{\abs{wx}>\frac{\lip}{2{\smooth}}}\br{\frac{\lip}{{\smooth}}wx - \frac{\lip^2}{4{\smooth}^2}}}.
\end{align*}
Note that the function above is 
convex,
$\lip$-Lipschitz and ${\smooth}$-smooth. The population risk is,
\begin{align*}
    {\lossF}(w;\cD) &= \frac{\lip}{2}wvp + 
    \frac{{\smooth}}{2} \mathbbm{1}_{\abs{w} \leq \frac{\lip}{2{\smooth}}} w^2
    +
    \frac{{\smooth}(1+vp)}{4}\br{\frac{\lip}{{\smooth}}w-\frac{\lip^2}{4{\smooth}^2}}\mathbbm{1}_{\abs{w} > \frac{\lip}{2{\smooth}}} \\&+  \frac{{\smooth}(1-vp)}{4} \br{-\frac{\lip}{{\smooth}}w-\frac{\lip^2}{4{\smooth}^2}}\mathbbm{1}_{\abs{w} > \frac{\lip}{2{\smooth}}} \\
    & = \frac{\lip}{2}wvp + 
    \frac{{\smooth}}{2} \mathbbm{1}_{\abs{w} \leq \frac{\lip}{2{\smooth}}} w^2
    -\frac{\lip^2}{8{\smooth}}\mathbbm{1}_{\abs{w} > \frac{\lip}{2{\smooth}}}
    + \frac{\lip vpw}{2}\mathbbm{1}_{\abs{w} > \frac{\lip}{2{\smooth}}} \\
\end{align*}
Let $\out$ be output some algorithm given $n$ i.i.d. samples from $\cD$. Consider two cases:
\paragraph{Case 1: } $\abs{\out} > \frac{\lip}{2{\smooth}}$: The gradient in this case is,
\begin{align*}
    \nabla {\lossF}(\out;\cD) = \frac{\lip vp}{2} + \frac{\lip vp}{2} = \lip vp
\end{align*}
Therefore, with the setting of $p = \frac{1}{\sqrt{n}}$, we have,
\begin{align*}
    \abs{\nabla {\lossF}(\out;\cD)}   = \abs{\lip vp} = \frac{\lip}{\sqrt{n}}.
\end{align*}
\paragraph{Case 2: } $\abs{\out} \leq  \frac{\lip}{2{\smooth}}$:
\begin{align*}
    \nabla {\lossF}(\out;\cD) = \frac{\lip vp}{2} + {\smooth}\out
\end{align*}
Suppose there exists 
an algorithm the output $\out$ of which, with $n$ samples guarantees that $\mathbb{E}\norm{\nabla {\lossF}(\out;\cD)} < o\br{\frac{\lip}{\sqrt{n}}}$. Then from Markov's inequality, with probability at least $0.9$, we have that $\norm{\nabla {\lossF}(\out;\cD)}^2 < o\br{\frac{\lip^2}{n}}$. Let $\tilde w =- \frac{2{\smooth}\out}{\lip}$, then we have that
with probability at least $0.9$,
\begin{align*}
    \norm{\nabla {\lossF}(\out;\cD)}^2 \leq o\br{\frac{\lip^2}{n}} \iff \abs{vp - \tilde w }^2 < o\br{\frac{1}{n}}
\end{align*}
This contradicts the well-known bias estimation lower bounds using Le Cam's method (\cite{duchi2016lecture}, Example 7.7), hence $\mathbb{E}\abs{\nabla {\lossF}(\out;\cD)}\geq \Omega\br{\frac{\lip}{\sqrt{n}}}$.
Combining the two cases finishes the proof.
\cnote{Minor thing: These lower bounds include some small constant factors, so I guess the contradiction is for $\mathbb{E}\norm{\nabla {\lossF}(\out;\cD)}^2 = o(\frac{\lip^2}{n})$, and the conclusion would be its negation, $\mathbb{E}\norm{\nabla {\lossF}(\out;\cD)}^2 = \Omega(\frac{\lip^2}{n})$. Does this suffice to also say $\mathbb{E}\norm{\nabla {\lossF}(\out;\cD)} = \Omega(\frac{\lip}{\sqrt n})$?}
\enayat{I changed the proof as you said and, and switched from expectation to constant probability lower bound. However I cannot find a reference to this now. I recall reading somewhere that the $1-\delta$ confidence lower bound is $\frac{\log{1/\delta}}{n}$, but cannot find it. }\cnote{I guess this is a direct consequence of Fano's inequality, but I couldn't find an explicit ref either.Perhaps in the lower bounds for SCO paper (Agarwal, Bartlett, Ravikumar, Wainwright)}
\enayat{I will try to add more details, may be after the deadline today unless you think that there is something that can break here}
}
\end{proof}

%% file: appendix/app-emp-sp.tex
\section{Missing Results for Empirical Stationary Points}

\subsection{Private Spiderboost} \label{app:spiderboost}

The following lemma largely follows from the analysis in \cite{spiderboost}. We present a full proof below for completeness. 
\begin{lemma} \label{lem:spiderboost-convergence}
Let the conditions of Lemma \ref{lem:spider-grad-err} be satisfied.
Let $\eta \leq \frac{1}{2\smooth}$ and $q \leq O\br{\frac{1}{\tau_2^2 \eta^2}}$. Then the output of Private SpiderBoost, $\out$ satisfies %
\begin{align}\label{eq:spider-acc}
    \ex{}{\norm{\nabla F(\out;S)}} = O\br{\sqrt{\frac{F_0}{\eta T}} + \tau_1}.
\end{align}
\end{lemma}
\begin{proof}
In the following, for any $t\in[T]$, let $s_t = \floor{\frac{t}{q}} q$ (i.e. the index corresponding to the start of the phase containing iteration $t$).

By a standard analysis for smooth functions we have (recalling that $\nabla_t$ is an unbiased estimate of $\nabla F(w_t;S)$ for any $t\in[T]$)
\begin{align*}
    \elossS{w_{t+1}} \leq \elossS{w_t} + \frac{\eta}{2}\norm{\nabla \elossS{w_t} - \grad_t}^2 - \br{\frac{\eta}{2} - \frac{\smooth \eta^2}{2}}\norm{\grad_t}^2.    
\end{align*}
Taking expectation we have the following manipulation using the update rule of Algorithm \ref{alg:spider}
\begin{align*}
    \ex{}{\elossS{w_{t+1}} - \elossS{w_{t}}} &\leq \frac{\eta}{2}\ex{}{\norm{\nabla \elossS{w_t} - \grad_t}^2} - \br{\frac{\eta}{2} - \frac{\smooth \eta^2}{2}}\ex{}{\norm{\grad_t}^2} \\
    &\leq \frac{\eta\tau_2^2}{2}\sum_{k=s_t+1}^t \ex{}{\norm{w_{k+1} - w_k}^2} + \frac{\eta}{2}\ex{}{\norm{\grad_{s_t} - \elossS{w_{s_t}}}^2} \\&\quad - \br{\frac{\eta}{2} - \frac{\smooth \eta^2}{2}}\ex{}{\norm{\grad_t}^2} \\
    &\leq \frac{\eta^3\tau_2^2}{2}\sum_{k=s_t+1}^t \ex{}{\norm{\grad_k}^2} + \frac{\eta\tau_1^2}{2} - \br{\frac{\eta}{2} - \frac{\smooth \eta^2}{2}}\ex{}{\norm{\grad_t}^2}, \\
\end{align*}

where the second inequality follows from Lemma \ref{lem:spider-grad-err} and the last inequality follows from the update rule. Note that if $t=s_t$ the sum is empty. 
Summing over a given phase we have %
\begin{align}
    \ex{}{\elossS{w_{t+1}} - \elossS{w_{s_t}}} &\leq \frac{\eta^3\tau_2^2}{2}\sum_{k=s_t}^t \sum_{j=s_t+1}^k \ex{}{\norm{\grad_j}^2} + \sum_{k=s_t}^t \textstyle\left[ \frac{\eta\tau_1^2}{2} - \br{\frac{\eta}{2} - \frac{\smooth \eta^2}{2}}\ex{}{\norm{\grad_k}^2}\right] \nonumber\\
    &\leq \frac{\eta^3\tau_2^2q}{2}\sum_{k=s_t}^t \ex{}{\norm{\grad_k}^2} + \sum_{k=s_t}^t \textstyle\left[ \frac{\eta\tau_1^2}{2} - \br{\frac{\eta}{2} - \frac{\smooth \eta^2}{2}}\ex{}{\norm{\grad_k}^2}\right] \nonumber \\
    &= -\sum_{k=s_t}^t\Bigg[ \underbrace{\br{\frac{\eta}{2} - \frac{\smooth \eta^2}{2} -\frac{\eta^3\tau_2^2q}{2} }}_{A}\ex{}{\norm{\grad_k}^2} - \frac{\eta\tau_1^2}{2}\Bigg], \label{eq:phase-loss}
\end{align}
where the second inequality comes from the fact that each gradient appears at most $q$ times in the sum. 
We now sum over all phases.
Let $P=\bc{p_0,p_1,...,} = \bc{0,q,2q,...,\floor{\frac{T-1}{q}}q,T}$.
We have
\begin{align*}
    \ex{}{\elossS{w_{T}} - \elossS{w_{0}}} \leq \sum_{i=1}^{|P|} \ex{}{\elossS{w_{p_i}} - \elossS{w_{p_{i-1}}}} \\
    \leq  -\sum_{t=0}^T A \, \ex{}{\norm{\grad_k}^2} + \frac{T\eta\tau_1^2}{2}.
\end{align*}
Rearranging the above yields
\begin{align}
    \frac{1}{T} \sum_{t=0}^T \ex{}{\norm{\grad_k}^2} \leq \frac{\fnot}{TA} + \frac{\eta\tau_1^2}{2A}. \label{eq:vrg-norm-bound}
\end{align}
Now let $i^*$ denote the index of $\out$ selected by the algorithm. Note that 
\begin{align} \label{eq:triangle-grad-bound}
    \ex{}{\norm{\nabla \elossS{w_{i^*}}}^2} \leq 2\ex{}{\norm{\nabla \elossS{w_{i^*}} - \grad_{i^*}}^2} + 2\ex{}{\norm{\grad_{i^*}}^2}.
\end{align}
The second term above can be bounded via inequality \eqref{eq:vrg-norm-bound}.
To bound the first term we have by Lemma \ref{lem:spider-grad-err} that 
\begin{align*}
    \ex{}{\norm{\grad_{i^*} - \nabla \elossS{w_{i^*}}}^2} &\leq \tau_2^2\sum_{k=s_{t^*}+1}^{t^*} \ex{}{\norm{w_k - w_{k-1}}^2} + \tau_1^2 \\
    &= \eta^2\tau_2^2\sum_{k=s_{t^*}+1}^{t^*} \ex{}{\norm{\grad_k}^2} + \tau_1^2 \\
    &\leq \frac{q\eta^2\tau_2^2}{T}\sum_{k=0}^{T} \ex{}{\norm{\grad_k}^2} + \tau_1^2 \\
    &\leq \frac{\tau_2^2 \eta^2 q \fnot}{TA} + \frac{\eta^3q\tau_2^2}{2A}\tau_1^2 + \tau_1^2,
\end{align*}
where the last inequality comes from inequality \eqref{eq:vrg-norm-bound}  %
and the expectation over $i^*$. Plugging into inequality \eqref{eq:triangle-grad-bound} one can obtain
\begin{align}\label{eq:A-bound}
    \ex{}{\norm{\nabla \elossS{w_{i^*}}}^2} \leq \frac{2\fnot}{TA}(1+\tau_2^2\eta^2q) + \br{\frac{\eta}{A} + 2 + \frac{\tau_2^2\eta^3q}{A}}\tau_1^2.
\end{align}
Now recall 
$A = \frac{\eta}{2} - \frac{\smooth\eta^2}{2} - \frac{\eta^3 \tau_2^2 q}{2}$.
Since $q \leq O\br{\frac{1}{\tau_2^2 \eta^2}}$ and $\eta \leq \frac{1}{2\smooth}$ we have $A=\Theta(\eta)$. Thus plugging into inequality \eqref{eq:A-bound} and again using the fact that $q \leq O\br{\frac{1}{\tau_2^2 \eta^2}}$ we have
\begin{align*}
    \ex{}{\norm{\nabla \elossS{w_{i^*}}}^2} &= O\br{\frac{\fnot}{T\eta}(1+\tau_2^2\eta^2q) + \br{3 + \frac{\tau_2^2\eta^3q}{A}}\tau_1^2} 
    = O\br{\frac{\fnot}{T\eta} + \tau_1^2}.
\end{align*}
The claim then follows from the Jensen inequality.
\end{proof}

For privacy, we will rely on the moments accountant analysis of \cite{Abadi16}. This roughly gives the same analysis as using privacy amplification via subsampling and the advanced composition theorem, but allows for improvements in log factors. We provide the following theorem implicit in \cite{Abadi16} Theorem 1 below. The same result can be obtained using the analysis for \cite{KLL21} Theorem 3.1 which uses the truncated central differential privacy guarantees of the Gaussian mechanism \cite{truncatedCDP}.
\begin{theorem}[\label{thm:moment-accountant}\cite{Abadi16,KLL21}]
Let $\varepsilon,\delta \in (0,1]$ and $c$ be a universal constant. Let $D\in\cY^n$ be a dataset over some domain $\cY$, and let $h_1,...,h_T:\cY\mapsto\re^d$ be a series of (possibly adaptive) queries such that for any $y\in\cY$, $t\in[T]$, $\norm{h_t(y)}_2 \leq \lambda_t$. Let $\sigma_t = \frac{c\lambda_t\sqrt{\log{1/\delta}}}{\varepsilon}\max\bc{\frac{1}{b},\frac{\sqrt{T}}{n}}$. Then the algorithm which samples batches of size $B_1,..,B_t$ of size $b$ uniformly at random and outputs $\frac{1}{n}\sum_{y\in B_t}h_t(y) + g_t$ for all $t\in[T]$ where $g_t \sim \cN(0,\mathbb{I_d}\sigma_t^2)$, is $(\varepsilon,\delta)$-DP.
\end{theorem}
We note that the original statement of the Theorem in \cite{Abadi16} requires $\sigma_t \geq \frac{c\lambda_t\sqrt{T \log{1/\delta}}}{n\varepsilon}$ and $T \geq \frac{n^2\varepsilon}{b^2}$ (or $T \geq \frac{n^2}{b^2}$ so long as $\varepsilon \leq 1$). However, in the case where $T \leq \frac{n^2}{b^2}$, one can simply consider the meta algorithm that does run $T' = \frac{n^2}{b^2}$ steps and only outputs the first $T$ results. This algorithm is at least as private as the algorithm which outputs every result, and under the setting $T'$ the scale of noise is $\frac{8\lambda_t\sqrt{\log{1/\delta}}}{b\varepsilon}$.

We can now prove the main result for Private Spiderboost, restated below. We note that the setting of $b_2$ given below will always be less than $n$ under required conditions. More details are provided in the proof below. 
\begin{theorem}[Private Spiderboost]
Let $n \geq \max\bc{\frac{(\lip\varepsilon)^2}{\fnot \smooth d\log{1/\delta}}, \frac{\sqrt{d}\max\bc{1,\sqrt{\smooth\fnot}/\lip}}{\varepsilon}}$.
Private Spiderboost run with parameter settings $\eta = \frac{1}{2\smooth}$, 
$b_1 = n$,
$b_2 = \floor{\max\bc{\br{\frac{\lip n\varepsilon}{\sqrt{\fnot \smooth d\log{1/\delta}}}}^{2/3}, \frac{(\lip nd\log{1/\delta})^{1/3}}{(\smooth F_0)^{1/6}\varepsilon^{2/3}}}}$,
$T = \floor{\max\bc{\br{\frac{(F_0 \smooth)^{1/4}n\varepsilon}{\sqrt{\lip d\log{1/\delta}}}}^{4/3}, \frac{n\varepsilon}{ \sqrt{d\log{1/\delta}}}}}$, and
$q = \floor{\frac{n^2\varepsilon^2}{\smooth^2 T d \log{1/\delta}}}$
satisfies
\begin{align*}
    \ex{}{\norm{\nabla F(\tilde{w})}} 
    &= O\br{\br{\frac{\sqrt{\fnot\smooth\lip d \log{1/\delta}}}{n\varepsilon}}^{2/3} + \frac{\sqrt{d\log{1/\delta}}\lip}{n\varepsilon} }
\end{align*}
is $(\varepsilon,\delta)$-DP and has oracle complexity $\tilde{O}\br{\max\bc{\br{\frac{n^{5/3}\varepsilon^{2/3}}{d^{1/3}}}, \br{\frac{n\varepsilon}{\sqrt{d}}}^2}}$.
\end{theorem}
\begin{proof}
For privacy, we rely on the moment accountant analysis of the Gaussian mechanism as per Theorem \ref{thm:moment-accountant}. 
Note that each gradient estimate computed in line~\ref{line:grad_oracle} has elements with $\ell_2$-norm at most $\lip$, and this estimate is computed at most $\frac{T}{q}$ times. Similarly, for a gradient variation at step $t$ in line \ref{line:gv_oracle} we have norm bound $\smooth\norm{w_{t}-w_{t-1}}$, and have that at most $T$ such estimates are computed. As such, the scale of noise in both cases ensures the overall algorithm is $(\varepsilon,\delta)$-DP by Theorem \ref{thm:moment-accountant}.

We now prove the convergence result. 
To simplify notation in the following, we define $\bar{\alpha}=\frac{\sqrt{d\log{1/\delta}}}{n\epsilon}$.
If $b_1=n$ (full batch gradient), the conditions of Lemma \ref{lem:spider-grad-err} are satisfied with
$\tau_1^2 = O\br{\frac{\lip^2 T \bar{\alpha}^2}{q}}$ and 
$\tau_2^2 = O\br{\frac{\smooth^2}{b_2} + \smooth^2 T\bar{\alpha}^2}$ and some setting of $q$ 
so long as $T \geq q\frac{n^2}{b_1^2}=q$ and $T \geq \frac{n^2}{b_2^2}$.
Further, if $b_2 \geq \frac{1}{T \bar{\alpha}^2}$ then $\tau_2^2 =O\br{\smooth^2 T \bar{\alpha}^2}$. Thus the condition on $q$ in Lemma \ref{lem:spiderboost-convergence} is satisfied with $q = \frac{\smooth^2}{\tau_2^2} =  \frac{1}{T \bar{\alpha}^2}$ since $\eta = \frac{1}{2\smooth}$ 

Plugging into Eqn. \eqref{eq:spider-acc} we obtain
\begin{align}
    \ex{}{\norm{\nabla F(\tilde{w})}} 
    &= O\br{\sqrt{\frac{\fnot\smooth}{T}} + \frac{\lip\sqrt{T}\bar{\alpha}}{\sqrt{q}}} \nonumber \\
    &= O\br{\sqrt{\frac{\fnot\smooth}{T}} + \lip T\bar{\alpha}^2}. \label{eq:acc_wrt_T}
\end{align}
We now consider the setting of $T$. 
Since $q = \frac{1}{T \bar{\alpha}^2}$, it suffices to set $T \geq \frac{1}{\bar{\alpha}}$ to ensure $T \geq q$.
We now set
$T = \max\bc{\br{\frac{(\smooth F_0)^{1/4}}{\sqrt{\lip}\bar{\alpha}}}^{4/3}, \frac{1}{\bar{\alpha}}}$. Using Eqn. \eqref{eq:acc_wrt_T} above we have
\begin{align*} 
    \ex{}{\norm{\nabla F(\tilde{w})}} 
    &= O\br{\br{\sqrt{\fnot\smooth\lip}\bar{\alpha}}^{2/3} + \lip\bar{\alpha}}.
\end{align*}
The claimed rate now follows if there exists a valid setting for $b_2$ satisfying the previously stated conditions. 
The restrictions on the batch size implied by $T$ imply we need 
$b_2 \geq \frac{n}{\sqrt{T}}$ and thus it suffices to have  $b_2 \geq \frac{\lip^{1/3} n\bar{\alpha}^{2/3}}{(\smooth F_0)^{1/6}}$ to satisfy this condition since $T \geq \br{\frac{(\smooth F_0)^{1/4}}{\sqrt{\lip} \bar{\alpha}}}^{4/3}$.
We recall that for the setting of $q$ to be valid we also require $b_2 \geq \frac{1}{T \bar{\alpha}^2}$ and because $T \geq \br{\frac{(\smooth F_0)^{1/4}}{\sqrt{\lip}\bar{\alpha}}}^{4/3}$ it suffices that $b_2 \geq \br{\frac{\lip}{\sqrt{\fnot \smooth}\bar{\alpha}}}^{2/3}$.
Thus we need
$b_2 = \max\bc{\br{\frac{\lip}{\sqrt{\fnot \smooth}\bar{\alpha}}}^{2/3}, \frac{\lip^{1/3} n\bar{\alpha}^{2/3}}{(\smooth F_0)^{1/6}}}$.
Finally, we need $b_2 \leq n$ whenever $q \geq 1$. Note that by the setting of $q$ and $T$ we have
$q \leq \br{\frac{\lip}{\sqrt{\fnot \smooth}\bar{\alpha}}}^{2/3}$ and thus
$q\geq 1 \implies \br{\frac{\sqrt{\smooth \fnot}\bar{\alpha}}{\lip}} \leq 1$. Under this same condition we have 
$\frac{\lip^{1/3} n\bar{\alpha}^{2/3}}{(\smooth F_0)^{1/6}} \leq n$. We further have 
$\br{\frac{\lip}{\sqrt{\fnot \smooth}\bar{\alpha}}}^{2/3} \leq n$ under the assumption $n \geq \frac{(\lip\varepsilon)^2}{\fnot \smooth d\log{1/\delta}}$ given in the theorem statement. It can also be verified that under the condition on $n$ given in the theorem statement that $q\geq1$. Thus the parameter settings obtain the claimed rate.

Note the number of gradient computations is bounded by
\begin{align*}
    O\br{Tb_2 + \frac{Tb_1}{q}} &= 
    \tilde{O}\br{ \br{\frac{n\varepsilon}{\sqrt{d}}}^{4/3}\max\bc{\br{\frac{n\varepsilon}{\sqrt{d}}}^{2/3}, \frac{(nd)^{1/3}}{\varepsilon^{2/3}}} + n\br{\frac{n\varepsilon}{\sqrt{d}}}^{2/3}} \\ 
    &= \tilde{O}\br{\max\bc{ \br{\frac{n\varepsilon}{\sqrt{d}}}^2, \frac{n^{5/3}\varepsilon^{2/3}}{d^{1/3}} }}.
\end{align*}

\end{proof}

\input{appendix/app-challenges}

%% file: appendix/app-challenges.tex
\subsection{Additional Discussion of Rate Improvement Challenges} \label{app:rate-challenges} 
We here give a more detailed version of the informal discussion in Section \ref{sec:lowerbound}.  
We want to emphasize that the goal of the following discussion is not to provide a universal lower bound, but rather to inform future research.

Let $\cL:\re^d\mapsto\re$ be a loss function.
We say the randomized mapping $\cO:\re^d\times(\re^d \cup \bot)\mapsto\re^d$, 
is a $(\tau_1,\tau_2)$-accurate oracle for $\cL$ if
$\forall w,w'\in\re^d$ 
\begin{align*} \label{eq:oracle-conditions}
    & \ex{\cO}{\cO(w,\bot)}=\nabla \cL(w) 
    , && \ex{\cO}{\cO(w,w')} = \nabla \cL(w) - \nabla \cL(w') \\
    & \ex{\cO}{\norm{\cO(w,\bot) - \nabla \cL(w)}^2} \leq \tau_1^2, && 
    \ex{\cO}{\norm{\cO(w,w')}^2} \leq \tau_2^2\norm{w-w'}^2. 
\end{align*}
In short, $\cO$ is an unbiased and accurate gradient/gradient variation oracle for $\cL$.
Define 
\begin{align*}
\mathfrak{m}(G,\smooth,\cL_0,\tau_1,\tau_2) = \inf_{\cA}\sup_{\cO,\cL}\inf\bc{\alpha: \ex{}{\norm{\nabla \cL(\cA(\cO,\smooth,\cL_0,\tau_1,\tau_2)}} \leq \alpha},
\end{align*}
where the supremum is taken over $\smooth$-smooth functions $\cL$ satisfying $\cL(0)-\argmin\limits_{w\in\re^d} \bc{\cL(w)} \leq \cL_0$, and $(\tau_1,\tau_2)$-accurate oracles for $\cL$. The infimum is taken over algorithms which make at most $G$ calls to $\cO$.

We have the following lower bound on $\mathfrak{m}$ (i.e. a lower bound on the accuracy of optimization algorithms which make at most $G$ queries to the oracle) following from \citep[Theorem 3]{arjevani2019lower} and the fact that the oracle model described above is a special case of the multi-query oracles considered by \cite{arjevani2019lower}.
\begin{theorem}[\label{thm:oracle-lb}\cite{arjevani2019lower}]
Let $G,\cL_0,\smooth,\tau_1,\tau_2\geq 0$ and define $\alpha = \br{\frac{\cL_0 \tau_2\tau_1}{G}}^{1/3} + \frac{\tau_1}{\sqrt{G}}$.
If $d = \tilde{\Omega}\br{\big[\frac{\cL_0\smooth}{\alpha^2}\big]^2}$, 
then
$\mathfrak{m}(G,\smooth,\cL_0,\tau_1,\tau_2) = \Omega\br{\alpha}$.
\end{theorem}

Now consider $\cL$ such that $\cL(w) = \frac{1}{n}\sum_{x\in S}\ell(w;x)$ for some $\lip$-Lipschitz and $\smooth$-smooth loss
$\ell:\re^d\times\cX\mapsto\re$ and $S\in\cX^n$. 
We are interested in designing some $(\hat{\tau}_1,\hat{\tau}_2)$-accurate and differentially private oracle, $\hat{\cO}$, which can then be used by an optimization algorithm, $\cA$, to obtain an approximate stationary point $\out=\cA(\cOdp,\smooth,\cL_0,\hat{\tau}_1,\hat{\tau}_2)$.
Specifically, we want $\cOdp$ to be capable of answering $G$ queries under $(\varepsilon,\delta)$-DP. A common method for achieving this is to ensure each query to $\cO$ is at least $(\frac{\varepsilon}{\sqrt{G}},\delta)$-DP and use advanced composition (or the more refined moment accountant) analysis. 
Such a setup encapsulates numerous results in the convex setting \cite{BFTT19,KLL21}, and is even more dominant in non-convex settings \cite{wang2017differentially,ZCHSB20,Abadi16}.

Our key observation is that under such a setup, any increase in the number of oracle calls to $G$ must be met with a proportional increase in the accuracy parameters $(\hat{\tau}_1,\hat{\tau}_2)$. Thus, if such an oracle, $\cOdp$ is applied in a black box fashion to a stochastic optimization algorithm $\cA$, one can obtain a lower bound on the accuracy of the overall algorithm independent of $G$.

Specifically, since estimating the gradient and gradient variation can be viewed as mean estimation problems on $n$ vectors, we can use fingerprinting code arguments to lower bound $\hat{\tau}_1$ and $\hat{\tau}_2$ \cite{SU15}. In Lemma \ref{lem:dp-oracle-acc-lb} below, we prove that any $(\hat{\tau}_1,\hat{\tau}_2)$-accurate oracle which ensures that any query is $(\frac{\varepsilon}{\sqrt{G}},\delta)$-DP must have 
$\hat{\tau}_1 = \Omega\Big(\frac{\lip\sqrt{G d\log{1/\delta}}}{n\varepsilon}\Big)$ and
$\hat{\tau}_2 = \Omega\Big(\frac{\smooth\sqrt{G d\log{1/\delta}}}{n\varepsilon}\Big)$.
Now, observe that by Theorem \ref{thm:oracle-lb}, we have
\begin{align*}
    \mathfrak{m}(G,\smooth,\cL_0,\hat{\tau}_1,\hat{\tau}_2) = \Omega\br{\br{\frac{\sqrt{\fnot\smooth\lip}\sqrt{d \log{1/\delta}}}{n\varepsilon}}^{2/3} + \frac{\lip\sqrt{d\log{1/\delta}}}{n\varepsilon}},
\end{align*}

which matches our upper bound.

We now remark on several ways the above barrier could be circumvented. 
The first and most obvious possibility is to employ a different privatization method than private oracles. However, this is particularly difficult in the nonconvex setting as existing methods which avoid private gradients (see e.g. \cite{feldman2020private} for several such methods) rely crucially on stability guarantees arising from convexity. Other possible ways to beat the above rate is by designing a stochastic optimization algorithm which leverages the structure of the noise used in private implementations of the oracle or makes use of additional assumptions 
to beat the $\Omega\br{\br{\frac{\cL_0 \tau_2\tau_1}{G}}^{1/3} + \frac{\tau_1}{\sqrt{G}}}$ non-private %
lower bound. 

\paragraph{Additional Details on Fingerprinting Bound}
We conclude by giving a concrete construction for the fingerprinting argument mentioned above.
\begin{lemma}\label{lem:dp-oracle-acc-lb}
Let $\lip,\smooth \geq 0$, $\varepsilon = O(1)$, $2^{-\Omega(n)} \leq \delta \leq \frac{1}{n^{1+\Omega(1)}}$ and $\sqrt{d\log{1/\delta}}/(n\varepsilon) = O(1)$. Let $\ell,\cL,S$ satisfy the assumptions above. Then there exists $\ell,S$ such that for any oracle, $\cO$, which is $(\tau_1,\tau_2)$-accurate for $\cL$ it holds that 
\begin{align*}
    \tau_1 = \Omega\br{\frac{\lip\sqrt{d\log{1/\delta}}}{n\varepsilon}} && \text{and} && \tau_2 = \Omega\br{\frac{\smooth\sqrt{d\log{1/\delta}}}{n\varepsilon}}.
\end{align*}
\end{lemma}
\begin{proof}
In the following, we use $u_j$ to denote the $j$'th component of some vector $u$.
Let $B=\frac{\lip}{\smooth\sqrt{d}}$ and define $h:\re \mapsto \re$ as 
\begin{align*}
    h(z) = \begin{cases}
    \frac{\smooth}{2}w^2 & \text{if} |w| \leq B\\
    \frac{\lip}{\sqrt{d}}|w| - \frac{\lip^2}{2d\smooth} & \text{ otherwise}
    \end{cases}
\end{align*}
Define $d'=\frac{d}{2}$ (assume $d$ is even for simplicity) and for any vector $u\in\re^d$ let $u^{(1)} = [u_1,...,u_{d'}]^{\top}$ and $u^{(2)} = [u_{d'+1},...,u_d]^{\top}$.
Define $\ell(w;x) = \ell_1(w;x) + \ell_2(w;x)$ where 
\begin{align*}
     \ell_1(w;x) = \frac{\lip}{\sqrt{d}}\ip{w^{(1)}}{x^{(1)}}, && \ell_2(w;x) = \frac{1}{2}\sum_{j=d'+1}^{d} h(w_j) x_j.
\end{align*}
Let $\cW = \bc{w: \norm{w}_{\infty} \leq B}$ and note for any $w \in \cW$ we have 
\begin{align*}
\nabla \ell(w;x) = [\frac{x_1}{\sqrt{d}},...,\frac{x_{d'}}{\sqrt{d}}, w_{d'+1} x_{d'+1}, ..., w_{d} x_{d}]^{\top}, && \nabla^2 \ell_2(w;x) = \smooth\cdot\mathsf{Diag}(0,...,0,x_{d'+1},...,x_d)
\end{align*}
That is, the Hessian of $\ell_2(w;x)$ is a diagonal matrix with entries from $x$. Thus one can observe that for any $x\in\bc{\pm 1}^d$ we have that $\ell(\cdot;x)$ is $\lip$-Lipschitz and $\smooth$-smooth over $\re^d$.

To prove a lower bound on $\tau_1$ and $\tau_2$, it suffices to show that for any $(\varepsilon,\delta)$-DP implementation of $\cO$ there exists $w\in\re^d$ such that $\ex{\cO}{\norm{\cO(w;\bot) - \nabla \cL(w)}^2} \geq \tau_1^2$ and there exist $w,w'\in\re^d$ such that 
$\ex{\cO}{\norm{\cO(w,w')}^2} \geq \tau_2^2\norm{w-w'}^2$. For sake of generality, we will show that these properties hold for a set of $w,w'$.

Note that to lower bound the gradient error, it suffices to lower bound the error with respect to the first $d'$ components. We thus argue using $\ell_1$, and will in fact show a lower bound for any $w\in\re^d$. Let $w\in\re^d$. We have for any $(\varepsilon,\delta)$-DP oracle $\cO$ there exists a dataset 
$S \subseteq \bc{\pm 1}^{d}$, where $|S|=n$,
of fingerprinting codes such that 
\begin{align*}
    \ex{\cO}{\norm{\cO(w;\bot) - \nabla \cL(w)}} \geq  \ex{\cO}{\norm{\cO(w;\bot)^{(1)} - \frac{1}{n}\sum_{x\in S} x^{(1)}}} = \Omega\br{\frac{\lip\sqrt{d\log{1/\delta}}}{n\varepsilon}}.
\end{align*}
The bound follows from standard fingerprinting code arguments. See \citep[Lemma 5.1]{bassily2014private} for a lower bound and \citep[Theorem 1.1]{SU15} for a group privacy reduction that obtains the additional $\sqrt{\log{1/\delta}}$ factor. This fingerprinting result also induces the parameter constraints in the theorem statement.
We thus have $\tau_1 = \Omega\br{\frac{\lip\sqrt{d\log{1/\delta}}}{n\varepsilon}}$.

Similarly, we will argue a bound on the gradient variation using $\ell_2$. Let $w,w'\in \cW$ and $u=(w-w')^{(2)}$. In what follows, we only use the second half of the components for each vector, and thus omit the superscript $^{(2)}$ from all vectors for readability.
We have
$\nabla \ell_2(w;x) - \nabla \ell_2(w';x) = \smooth[u_1 x_1,...,u_{d'}x_{d'}]^{\top}$. 
Then for any $c\in(0,\frac{2\lip}{\smooth\sqrt{d}}]$ and $u \in \bc{\pm c}^2$ we have 
\begin{align*}
    \ex{\cO}{\norm{\cO(w,w') - (\nabla \cL(w) - \nabla \cL(w'))}^2} &= \smooth^2\cdot\ex{\cO}{\sum_{j=1}^{d'} \br{\cO(w,w')_j - \frac{u_j}{n}\sum_{x\in S} x_j}^2 } \\
    &= \smooth^2\cdot\ex{\cO}{\sum_{j=1}^{d'} \br{u_j\Big(\frac{\cO(w,w')_j}{u_j} - \frac{1}{n}\sum_{x\in S} x_j \Big)}^2 } \\
    &= \smooth^2\cdot\ex{\cO}{c^2 \sum_{j=1}^{d'} \br{\frac{\cO(w,w')_j}{u_j} - \frac{1}{n}\sum_{x\in S} x_j }^2 } \\
    &= \Omega\br{\smooth^2c^2\frac{d^2\log{1/\delta}}{n^2\varepsilon^2} },
\end{align*}
where the last step again comes from fingerprinting results. Note that the extra factor of $d$ as compared to the previous bound comes from the fact that we are considering fingerprinting codes with norm larger by a factor of $\sqrt{d}$. We also use the fact that the vector $\cO(w,w')$ transformed using $u$ is $(\varepsilon,\delta)$-DP by post processing.
Now since $c = \frac{\norm{w-w'}}{\sqrt{d}}$ we have
\begin{align*}
    \ex{\cO}{\norm{\cO(w,w') - (\nabla \cL(w) - \nabla \cL(w'))}} &= \br{\smooth\norm{w-w'}\frac{\sqrt{d\log{1/\delta}}}{n\varepsilon}}.
\end{align*}
Finally, 
noting that $\ex{\cO}{\norm{\cO(w,w') - (\nabla \cL(w) - \nabla \cL(w'))}^2} \leq \ex{\cO}{\norm{\cO(w,w')}^2}$
we obtain $\tau_2 = \Omega\Big(\frac{\smooth\sqrt{d\log{1/\delta}}}{n\varepsilon}\Big)$.
This completes the proof.
\end{proof}
We remark that the accuracy lower bound for the gradient variation can hold for a much more general set of vectors than that given in the proof. Specifically, the same result can be obtained for any $u = w-w'$ such that $u$ has $\Theta(d)$ components which are $\Omega\big(\frac{\norm{u}}{\sqrt{d}}\big)$ (i.e. any sufficiently spread out vector). 
This uses the fact that it suffices to bound the number of components which disagree in sign with the fingerprinting mean and that fingerprinting codes are sampled using a product distribution, and thus the tracing attack used by fingerprinting constructions holds over any sufficiently large subset of dimensions. %

%% file: appendix/app-pop-sp.tex
\section{Missing Results for Population Stationary Points}\label{sec:app_pop}

Here we present the proof of privacy and accuracy for Algorithm \ref{Alg:tree_private_spider}. We start by proving the privacy guarantee. 
\begin{proof}[Proof of Theorem \ref{thm:privacy_tree_private_spider}]
By parallel composition of differential privacy, and since the used batches are disjoint, it suffices to prove  that each step in lines \ref{alg:SPIDER_first_grad} and \ref{alg:SPIDER_rec_grad} of the algorithm is $(\varepsilon, \delta)$-DP. 
Note that the gradient estimator in step \ref{alg:SPIDER_first_grad} has $\ell_2$-sensitivity $2L_0/b$, so by the Gaussian mechanism this step is $(\varepsilon, \delta)$-DP. 

For step \ref{alg:SPIDER_rec_grad}, suppose $S_{t,s} $ and $S_{t,s}'$ are neighboring datasets that differ in at most one element: $x_{i^*} \neq x_{i^*}'$, and let $\eta_{t,s_i}$ and $\eta_{t,s_i}^{\prime}$ the respective stepsizes used in step \ref{alg:SPIDER_step}. Then 
\[
\|\Delta_{t, s} - \Delta_{t, s}'\| = \frac{2^{|s|}}{b} \|\nabla f\left(w_{t, s}; x_{i^{*}}\right)-\nabla f\left(w_{t, \hat{s}}; x_{i^{*}}\right)-\left(\nabla f\left(w_{t, s}; x_{i *}^{\prime}\right)-\nabla f\left(w_{t, \hat{s}}; x_{i^{*}}^{\prime}\right)\right)\| \,,
\]

and note between the parent node $u_{t,\hat{s}}$ and $u_{t,s}$ there are $2^{D-|s|}$ iterates generated by the algorithm, which we denote as $w_{t,\hat{s}} = w_{t,s_0}, w_{t,s_1},...,w_{t,s_{2^{|D|-s}}} = w_{t,s}$. Then, by smoothness of $f$ and the triangle inequality
\begin{align*}
    &\|\Delta_{t, s} - \Delta_{t, s}'\| \\
    &= \frac{2^{|s|}}{b} \|\nabla f\left(w_{t, s}; z_{i^{*}}\right)-\nabla f\left(w_{t, \hat{s}}; z_{i^{*}}\right)-\left(\nabla f\left(w_{t, s}; z_{i *}^{\prime}\right)-\nabla f\left(w_{t, \hat{s}}; z_{i^{*}}^{\prime}\right)\right)\|\\
    &\leq \sum_{i = 1}^{2^{D-|s|}} \frac{2^{|s|}}{b} \left[\|\nabla f\left(w_{t, s_i}; z_{i^{*}}\right)-\nabla f\left(w_{t, s_{i-1}}; z_{i^{*}}\right)\| + \|\left(\nabla f\left(w_{t, s_i}; z_{i *}^{\prime}\right)-\nabla f\left(w_{t, s_{i-1}}; z_{i^{*}}^{\prime}\right)\right)\|\right]\\
    &\leq \sum_{i = 1}^{2^{D-|s|}} \frac{2^{|s|}}{b}L_1\eta_{t,s_{i-1}}\|\nabla_{t, s_{i-1}}\| + \sum_{i = 1}^{2^{D-|s|}} \frac{2^{|s|}}{b}L_1\eta_{t,s_{i-1}}'\|\nabla_{t, s_{i-1}}'\|\\
    &= 2\sum_{i = 1}^{2^{D-|s|}} \frac{2^{|s|}}{b}\frac{\beta}{2^{D/2}} = \frac{2\beta 2^{D/2}}{b}.
\end{align*}

The Gaussian mechanism combined with our choice of $\sigma_{t,s}$ certifies privacy of this step. 
\end{proof}

To prove Theorem \ref{thm:accuracy_tree_private_Spider}  we will need some technical lemmas. Define $(\mathcal{T}, \mathcal{S})$ as a random stopping time that indicates when Algorithm \ref{Alg:tree_private_spider} ends. Also, we say $(t_1, s_1) \preceq_2 (t_2,s_2)$ whenever $w_{t_1,s_1}$ comes before $w_{t_2,s_2}$ in the algorithm iterates.

\begin{lemma}[Gradient estimation error, extension of Lemma $6$ in \cite{fang2018spider}]\label{lem:tree_estim_err} %
Let $p\in (0,1)$. Then, with probability $1-p$ the event
\[\mathcal{E} = \{\|\nabla_{t,s} - \nabla F(w_{t,s};{\cal D})\|^2 \leq \alpha \cdot \Tilde{\alpha} \quad \forall (t,s) \preceq_2 (\mathcal{T}, \mathcal{S})\}\] 

holds, under the parameter setting of $\sigma_{t,\varnothing}, \sigma_{t,s}$ and $\eta_{t,s}$ in Algorithm \ref{Alg:tree_private_spider}, for 
\[\alpha^2 \ge \left(\frac{L_0^2}{b} + \frac{\beta^2D2^D}{b}\right)\max\left\{1, \frac{(d+1)}{b \varepsilon^{2}}\right\}\quad\text{ and }\quad\tilde{\alpha} \ge 256\log{\frac{1.25}{\delta}}\log{\frac{2T2^{D+1}}{p}}\alpha.\]
\end{lemma}

\begin{proof} 

Recall the gradient estimate associated to a left child node is the same as that of the parent node. Hence, the gradient estimate of a non-leaf node is the same as that of the left-most leaf of its left sub-tree. In addition, we only need to control the gradient estimation error when we perform a gradient step, which occurs at the leaves. Then, to prove the claim, it suffices to prove that we can control the gradient estimation error at the leaves. Since, the number of iterations (and leaves) is at most $T2^{D-1}$, to prove event $\mathcal{E}$ happens with probability $1-p$, by the union bound it suffices to prove that $\mathbb{P}[\|\nabla_{t,s} - \nabla F(w_{t,s};{\cal D})\|^2 > \alpha\cdot\Tilde{\alpha}] \leq \frac{p}{T2^{D-1}}$ for every $(t,s) \preceq_2 (\mathcal{T}, \mathcal{S})$ where $u_{t,s}$ is a leaf. 

Denote by $\mathcal{F}_{t}$ the sigma algebra generated by randomness in the algorithm until the end of round $t$. Fix $(t,s) \preceq_2 (\mathcal{T}, \mathcal{S})$ such that $u_{t,s}$ is leaf, and let $u_{t,s_{\varnothing}} = u_{t,s_0}, u_{t,s_1}, ..., u_{t,s_k} = u_{t,s}$ be the path from the root to $s$. Next, extract a sub-sequence of it including only the root and the nodes that are right children, obtaining $u_{t,s_{\varnothing}} = u_{t,s_{a_0}}, u_{t,s_{a_1}}, ..., u_{t,s_{a_m}} = u_{t,s}$. Now we can write 
\begin{align*}
    &\nabla_{t,s} - \nabla F(w_{t,s};{\cal D}) = \sum_{i = 0}^m g_{t, s_{a_i}} + \sum_{x \in S_{t,\varnothing}} \underbrace{\frac{1}{b} \left(\nabla f(w_{t,\varnothing}; x) - \nabla F(w_{t,\varnothing};{\cal D})\right)}_{\gamma_{1,x}}\\
    &+ \sum_{i = 1}^m\sum_{x \in S_{t,s_{a_i}}} \!\!\!\underbrace{\frac{2^{|s_{a_i}|}}{b} \!\left[\left(\nabla f(w_{t,s_{a_i}}; x) \! - \!\nabla f(w_{t,s_{a_{i-1}}}; x)\right) \!-\! \left( \nabla F(w_{t,s_{a_i}};{\cal D})\! -\! \nabla F(w_{t,s_{a_{i-1}}};{\cal D})\right)\right]}_{\gamma_{2,x,i}}.
\end{align*}

To bound the estimation error, we note that

\begin{align*}
    &\mathbb{P}[\|\nabla_{t,s} - \nabla F(w_{t,s};{\cal D})\|^2 > \alpha \cdot \Tilde{\alpha} | \mathcal{F}_{t-1}] \\
    &\leq \mathbb{P}\Big[\Big\| \sum_{i = 0}^m g_{t, s_{a_i}}\Big\|^2 > \frac{\alpha \cdot \Tilde{\alpha}}{4}\Big| \mathcal{F}_{t-1}\Big]
     + \mathbb{P}\Big[\Big\| \sum_{x \in S_{t,\varnothing}} \gamma_{1,x} + \sum_{i = 1}^m\sum_{x \in S_{t,s_{a_i}}} \gamma_{2,x,i}\Big\|^2 > \frac{\alpha \cdot \Tilde{\alpha}}{4} \Big| \mathcal{F}_{t-1}\Big].
\end{align*} 

and proceed to bound each term on the right hand side separately. 
By vector subgaussian concentration (see Lemma 1 in \cite{jin2019short}) %
and noting that the gaussians are independent of $\mathcal{F}_{t-1}$, we know that 
\[\mathbb{P}\left[\left\| \sum_{i = 0}^m g_{t, s_{a_i}}\right\|^2 > \frac{\alpha \cdot \Tilde{\alpha}}{4}\right] \leq 4^d \exp{-\frac{\alpha \cdot \Tilde{\alpha}}{32(\sigma_{t,\varnothing}^2 + \sum_{i = 1}^m\sigma_{t,s_{a_i}}^2)}} \,,
\]
and in order to bound this probability by $\frac{p}{2T2^{D-1}}$, since $m \leq D$, it suffices that

\begin{align*}
    \alpha\cdot\tilde{\alpha} &> 32\log{\frac{4^dT2^{D}}{p}}\left[\frac{8 L_{0}^{2} \log{1.25 / \delta}}{b^{2} \varepsilon^{2}} + \frac{8D2^D\beta^2\log{1.25 / \delta}}{b^{2} \varepsilon^{2}}\right]\\
    &= 256\log{\frac{1.25}{\delta}}\left[d\log{4} + \log{\frac{T2^{D}}{p}}\right]\left[\frac{L_{0}^{2}}{b^{2} \varepsilon^{2}} + \frac{D2^D\beta^2}{b^{2} \varepsilon^{2}}\right].
\end{align*}
    
Now, noting that surely %
\[\|\gamma_{1,x}\| \leq \frac{2L_0}{b}\quad \text{ and }\quad \|\gamma_{2,x,i}\| \leq \frac{2\beta 2^{D/2}}{b},\]
where the second bound comes from following similar steps as in the 
privacy analysis in Theorem \ref{thm:privacy_tree_private_spider}, we have that $\sum_{x \in S_{t,\varnothing}} \gamma_{1,x} + \sum_{i = 1}^m\sum_{x \in S_{t,s_{a_i}}} \gamma_{2,x,i}$ is a sum of bounded martingale differences when conditioned on $\mathcal{F}_{t-1}$, thus by %
concentration of martingale-difference sequences in $\ell_2$ (see Proposition 2 in \cite{fang2018spider}), and using the fact that $|S_{t,\varnothing}| = b$ and $|S_{t, s_{a_i}}| = b/2^{|s_{a_i}|}$ it follows that

\[\mathbb{P}\left[\left\| \sum_{x \in S_{t,\varnothing}} \gamma_{1,x} + \sum_{i = 1}^m\sum_{x \in S_{t,s_{a_i}}} \gamma_{2,x,i}\right\|^2 > \frac{\alpha \cdot \Tilde{\alpha}}{4}\mid \mathcal{F}_{t-1}\right] \leq 4\exp{-\frac{\alpha\cdot\tilde{\alpha}}{16\left[\frac{4L_0^2}{b} + \sum_{i = 1}^m\frac{4\beta^22^{D}}{2^{|s_{a_i}|}b}\right]}}.\]
    
Repeating a similar argument as before, to bound this term by $\frac{p}{2T2^{D-1}}$, it suffices that  
\[\alpha \cdot \Tilde{\alpha} \ge 64\log{\frac{2T2^{D+1}}{p}} \left[\frac{L_0^2}{b} + \frac{\beta^2D2^D}{b}\right].\]

Finally, both conditions hold simultaneously for %
\[\alpha^2 \ge \left(\frac{L_0^2}{b} + \frac{\beta^2D2^D}{b}\right)\max\left\{1, \frac{(d+1)}{b \varepsilon^{2}}\right\}\]
and 
\[\tilde{\alpha} \ge 256\log{\frac{1.25}{\delta}}\log{\frac{2T2^{D+1}}{p}}\alpha.\]
\end{proof}

\begin{lemma}[Descent lemma; Lemma $7$ in \cite{fang2018spider}]\label{lem:descent_lemma} Under the assumption that the event $\mathcal{E}$ from Lemma \ref{lem:tree_estim_err} occurs and $\beta \leq 2^{D/2}\tilde{\alpha}$, we have that if Algorithm \ref{Alg:tree_private_spider} reaches the last line, then 

\[F(w_{T,\ell(2^D)};{\cal D}) - F(0;{\cal D}) \leq -(T2^{D-1})\frac{\beta \cdot \Tilde{\alpha}}{4\cdot2^{D/2}L_1}.\]

where $w_{T,\ell(2^D)}$ is the last iterate in the $T$-th tree of Algorithm \ref{Alg:tree_private_spider}.
\end{lemma} %
We provide the proof of Lemma \ref{lem:descent_lemma} adapted to our case for completeness. %

\begin{proof}

By standard analysis for smooth functions we have %

\[F(w_{t,s^+};{\cal D}) \leq F(w_{t,s};{\cal D}) - \frac{\eta_{t,s}}{2} (1-\eta_{t,s} L_1)\|\nabla_{t,s}\|^2 + \frac{\eta_{t,s}}{2}\|\nabla_{t,s} - \nabla F(w_{t,s};{\cal D})\|^2,\]

where $\eta_{t,s} = \frac{\beta}{2^{D/2}L_1 \|\nabla_{t,s}\|}$ and $u_{t,s^+}$ is the node after $u_{t,s}$ in the tree. Since $\beta \leq 2^{D/2}\tilde{\alpha}$ and $\|\nabla_{t,s}\| > 2\tilde{\alpha}$, we have that $(1-\eta_{t,s} L_1) \ge 1/2$. Using this inequality, the definition of $\eta_{t,s}$ and the fact that we are assuming $\cE$ occurs, we obtain 

\begin{align*}
    F(w_{t,s^+};{\cal D}) - F(w_{t,s};{\cal D}) &\leq - \frac{\beta}{4\cdot 2^{D/2}L_1 \|\nabla_{t,s}\|}\|\nabla_{t,s}\|^2 + \frac{\beta}{2\cdot 2^{D/2}L_1 \|\nabla_{t,s}\|}\alpha\cdot\tilde{\alpha}\\
    &\leq - \frac{\beta}{4\cdot 2^{D/2}L_1}\cdot\tilde{\alpha},
\end{align*}

where the second inequality comes from $\|\nabla_{t,s}\|>2\tilde{\alpha}$ and $\alpha\leq\tilde{\alpha}$. Then telescoping over all $T2^{D-1}$ iterations provides the claimed bound. 

\end{proof}

We are now ready to prove the convergence guarantee of Algorithm \ref{Alg:tree_private_spider}. 

\begin{proof}[Proof of Theorem \ref{thm:accuracy_tree_private_Spider}]

From Lemma \ref{lem:tree_estim_err}, we know that $\|\nabla_{t,s} - \nabla F(w_{t,s}; {\cal D})\|^2 \leq \alpha\cdot\tilde{\alpha}$ with probability $1-p$ when 
\[\textstyle\alpha = \sqrt{2}L_0\max\bc{\frac{1}{n^{1/3}}, \left(\frac{\sqrt{d}}{n\varepsilon}\right)^{1/2}}, \tilde{\alpha} =  \left(256\log{\frac{1.25}{\delta}}\log{\frac{2T2^{D+1}}{p}} + \frac{8L_1F_0\sqrt{2D}(D/2+1)}{2L_0^2}\right)\alpha.\]

Indeed, using our parameter setting, and noting that $d > b\varepsilon^2$ if %
and only if, $d > n^{2/3}\varepsilon^2$, yields 
\begin{align*}
    \alpha^2 &\ge  \frac{L_0^2}{b} \max\left\{1, \frac{(d+1)}{b \varepsilon^{2}}\right\} + \frac{\beta^2}{2}\max\left\{1, \frac{(d+1)}{b \varepsilon^{2}}\right\}\\
    &= L_0^2 \left(\frac{1}{n^{2/3}}\mathbbm{1}_{\{d + 1 \leq n^{2/3}\varepsilon^2\}} + \frac{\sqrt{d}}{n\varepsilon}\mathbbm{1}_{\{d + 1 > n^{2/3}\varepsilon^2\}}\right) + \frac{\alpha^2}{2}\min\left\{1, \frac{b\varepsilon^2}{d}\right\}\max\left\{1, \frac{(d+1)}{b \varepsilon^{2}}\right\}\\
    &\ge L_0^2\max\left\{\frac{1}{n^{2/3}}, \frac{\sqrt{d}}{n\varepsilon}\right\} + \frac{\alpha^2}{2},
\end{align*}

which shows our values of $\alpha$ and $\tilde{\alpha}$ are valid for controlling the gradient estimation error with high probability, as claimed in Lemma \ref{lem:tree_estim_err}. 

Now, suppose for the sake of contradiction that Algorithm \ref{Alg:tree_private_spider} does not end in line \ref{alg:SPIDER_return} under $\cE$. This means it performs $T2^{D-1}$ gradient updates. We'll show this implies $(T2^{D-1})\frac{\beta \cdot\tilde{\alpha}}{4\cdot2^{D/2}L_1} > F_0$ and thus contradicts Lemma \ref{lem:descent_lemma}, which claims that $F_0 \geq -[F(w_{T,\ell(2^D)};{\cal D}) - F(w_{0,\ell(2^D)};{\cal D})] \geq (T2^{D-1})\frac{\beta \cdot \Tilde{\alpha}}{4\cdot2^{D/2}L_1}$. Indeed, note that by our parameter setting: 

\begin{align*}
    (T2^{D-1})\frac{\beta \cdot\tilde{\alpha}}{4\cdot2^{D/2}L_1} > F_0 &\iff \beta\cdot\tilde{\alpha} >  \frac{8L_1F_0}{T2^{D/2}}\\
    &\iff \alpha\min\bc{1, \frac{\sqrt{b}\varepsilon}{\sqrt{d}}}\cdot\tilde{\alpha} > \frac{8L_1F_0\sqrt{2D}}{T\sqrt{b}}\\
    &\iff \alpha\cdot\tilde{\alpha} > \frac{8L_1F_0\sqrt{2D}(D/2+1)\sqrt{b}}{n}\max\bc{1, \frac{\sqrt{d}}{\sqrt{b}\varepsilon}}\\
    &\iff \alpha\cdot\tilde{\alpha} > 8L_1F_0\sqrt{2D}(D/2+1)\max\bc{\frac{\sqrt{b}}{n}, \frac{\sqrt{d}}{n\varepsilon}},
\end{align*}

and noting that by the setting of $b$ we have $\max\left\{\frac{\sqrt{b}}{n},\frac{\sqrt{d}}{n\varepsilon}\right\} = \max\bc{\frac{1}{n^{2/3}}, \frac{\sqrt{d}}{n\varepsilon}}$, we conclude the following

\begin{align*}
    (T2^{D-1})\frac{\beta \cdot\tilde{\alpha}}{4\cdot2^{D/2}L_1} > F_0 &\iff \alpha\cdot\tilde{\alpha} > 8L_1F_0\sqrt{2D}(D/2+1)\max\bc{\frac{1}{n^{2/3}}, \frac{\sqrt{d}}{n\varepsilon}}\\
    &\iff \alpha\cdot\tilde{\alpha} > \frac{8L_1F_0\sqrt{2D}(D/2+1)}{2L_0^2}\alpha^2.
\end{align*}

Finally, note $\alpha\cdot\tilde{\alpha} =  \left(256\log{1.25/\delta}\log{2T2^{D+1}/p} + \frac{8L_1F_0\sqrt{2D}(D/2+1)}{2L_0^2}\right)\alpha^2$ and thus the last inequality holds under our parameter setting. Since this is equivalent to $(T2^{D-1})\frac{\beta \cdot\tilde{\alpha}}{4\cdot2^{D/2}L_1} > F_0$, we are done with the contradiction. It follows that with high probability, Algorithm \ref{Alg:tree_private_spider} ends in line \ref{alg:SPIDER_return} returning $w_{t,s}$ such that $\|\nabla_{t,s} \| \leq 2\tilde{\alpha}$. Also, by Lemma \ref{lem:tree_estim_err} we have 
 $\|\nabla F(w_{t,s}; {\cal D}) - \nabla_{t,s}\| < \tilde{\alpha}$, so the returned iterate satisfies by the triangle inequality
\[\|\nabla F(w_{t,s}; {\cal D})\| < 3\tilde{\alpha}.\]

In addition, the linear time oracle complexity follows from the fact that at each binary tree we use $b$ samples at the root, and then $b/2$ in levels $1$ to $D$. This gives a total of $b(D/2+1)$ samples used at every round. Since we run the algorithm for $T = \frac{n}{b(D/2+1)}$ rounds, we compute exactly $n$ gradients. To conclude, note the condition $n \ge \max\{\sqrt{d}(D/2+1)^2/\varepsilon, (D/2+1)^3\}$ implies the number of rounds $T$ is at least $1$. Besides, since the definition of $D$ implies $2^D < b$, the size of the mini-batches are well-defined (meaning Algorithm \ref{Alg:tree_private_spider} uses batches with at least $1$ sample). This concludes the proof.
\end{proof}

%% file: appendix/app-convex.tex
\section{Missing Results for Stationary Points in the Convex Setting}
\label{sec:app_convex}

We first give pseudo-codes of algorithms used in the section.

\begin{algorithm}[h]
\caption{Phased SGD$(S,  (w,x)\mapsto f(w;x)),R,\eta,\cS(\cdot), \sigma)$}
\label{alg:phased-sgd}
\begin{algorithmic}[1]
\REQUIRE Dataset $S$, 
loss function $f(\cdot;x))$, radius $R$ of the constraint set $\cW$, steps $T$, $\eta$, Selection function $\cS$, Noise variance $\sigma$ 
\STATE $w_{1} = 0$
\STATE $K = \lceil \log{\abs{S}}\rceil$ and $T_0 =1$
\FOR{$k=1$ to $K-1$}
\STATE $T_k = 2^{-k}\abs{S}, \eta_k = 4^{-k}\eta, \sigma_k = \eta_k\sigma$
\STATE $w_{k+1} = \mathsf{OutputPerturbedSGD}(w_k, S_{T_{k-1}+1:T_{k}}, R, \eta_k, \sigma_k, \cS(\cdot))$

\ENDFOR
\ENSURE{$\out = w_{K}$}
\end{algorithmic}
\end{algorithm}

\begin{algorithm}[h]
\caption{OutputPerturbedSGD$(w_1, S, (w,x)\mapsto f(w;x), \Delta(\cdot),R,\eta,\cS(\cdot)$}
\label{alg:psgd-output-perturbation}
\begin{algorithmic}[1]
\REQUIRE Dataset $S$, 
loss function $f(\cdot;x))$, regularizer $\Delta(\cdot)$, radius $R$ of the constraint set $\cW$, steps $T$, $\eta$, Selection function $\cS$, Noise variance $\sigma$ 
\FOR{$t=1$ to $\abs{S}-1$}
\STATE $w_{t+1} = \Pi_{\cW}\br{w_t-\eta\br{\nabla f (w_t;x_t)}} $
\ENDFOR
\STATE $\xi \sim \cN(0,\sigma^2\bbI)$
\STATE $\tilde w = \cS\br{\bc{w_t}_{t=1}^{\abs{S}}} $
\ENSURE{$\out = \tilde w+ \xi$ }
\end{algorithmic}
\end{algorithm}

\begin{algorithm}[h]
\caption{Noisy GD$(S, (w,x)\mapsto f(w;x)),R,T,\eta,\cS(\cdot), \sigma)$}
\label{alg:noisyGD}
\begin{algorithmic}[1]
\REQUIRE Dataset $S$, 
loss function $(w,x) \mapsto f(w;x)$, radius $R$ of the constraint set $\cW$, steps $T$, $\eta$, Selection function $\cS$, Noise variance $\sigma$ 
\STATE $w_1 = 0$
\FOR{$t=1$ to $T-1$}
\STATE $\xi_t \sim \cN(0,\sigma^2\bbI)$
\STATE $w_{t+1} = \Pi_{\cW}\br{w_t-\eta\br{\nabla F(w_t;S)+\xi_t}} $
\ENDFOR
\ENSURE{$\out = \cS\br{\bc{w_t}_{t=1}^T}$}
\end{algorithmic}
\end{algorithm}

\begin{proof}[Proof of Theorem \ref{thm:population-convex-gradient-small}]

The privacy guarantee, in both cases, follows from the privacy guarantees of Algorithm \ref{alg:noisyGD} and Algorithm \ref{alg:phased-sgd}, in Lemmas \ref{lem:population-dp-strongly-convex-composite} and \ref{lem:convex-phased-sgd} respectively, together with parallel composition.

We now proceed to the utility part. 
For simplicity of notation, let  $R=\norm{w^*}$. 
Recall the definition of the regularized losses $f^{(t)}(w,x)$ in Algorithm \ref{alg:noisyRR}. 
Let $\bc{\alpha_t}_t$ be such that $\mathbb{E}[\lossF^{(t-1)}(\out_t;\cD)] - \lossF^{(t-1)}(w_{t-1}^*;\cD) \leq \alpha_t$
where
$\out_t$ are the iterates produced in the algorithm and 
$w_{t-1}^* = \argmin_{w\in \bbR^d}\lossF^{(t-1)}(w;\cD)$.
Following \cite{allen2018make, foster2019complexity}, we first establish a general result which will be useful for both parts of the result.%

\begin{align*}
    \mathbb{E}\norm{\nabla \lossF(\out_{T};\cD)} &= \mathbb{E}\norm{\nabla \lossF^{(T-1)}(\out_T;\cD) + \lambda \sum_{t=0}^T2^t\br{\out_t - \out_T}} \\
    &\leq \mathbb{E}\norm{\nabla \lossF^{(T-1)}(\out_T;\cD) }+ \lambda \sum_{t=0}^{T-1}2^t\mathbb{E}\br{\norm{\out_t-w_{T-1}^*}+\norm{\out_T -  w^*_{T-1}}} \\
    &\leq 2\mathbb{E}\norm{\nabla \lossF^{(T-1)}(\out_T;\cD) } + \lambda\sum_{t=1}^{T-1}2^t\mathbb{E}\norm{\out_t-w_{T-1}^*} + \lambda\mathbb{E}\norm{w_0 - w_{T-1}^*}\\
    &\leq 2\mathbb{E}\norm{\nabla \lossF^{(T-1)}(\out_T;\cD) } + 4\sum_{t=1}^{T-1}\sqrt{\lambda 2^t\alpha_t} + \lambda R_{T-1} \\
    & \leq
    4\sqrt{\smooth\alpha_T}+
    4\sum_{t=1}^{T-1}\sqrt{\lambda 2^{t+1}\alpha_t} + \lambda 2^{T/2}R\\
      & \leq 4\sum_{t=1}^{T}\sqrt{\lambda 2^{t+1}\alpha_t} +
    \sqrt{\lambda \smooth}R
\end{align*}
where the third and fourth inequality follows from strong convexity of $\lossF^{(T-1)}(\cdot;\cD)$ and Lemma \ref{lem:RR-key-lemma} respectively.
The last inequality follows from the setting of $T$ since we have that $\lossF^{(T-1)}$ is $\smooth + \sum_{t=1}^{T-1}2^t\lambda \leq \smooth+\lambda 2^{T} \leq 2\smooth$ smooth.
Note that the definition of $R_t$ and Lemma \ref{lem:radius_bound}, 
$\norm{w^*_{T-1}}\leq R_{T-1}$, so the unconstrained minimizer lies in the constraint set. 
Therefore $\mathbb{E}\norm{\nabla \lossF^{(T-1)}(\out_T;\cD) } = \mathbb{E}\norm{\nabla \lossF^{(T-1)}(\out_T;\cD) - \nabla \lossF^{(T-1)}( w_{T-1}^*;\cD)} \leq 2\sqrt{\smooth \alpha_T}$.

Observe that from the setting of $T$, $\lossF^{(T)}$ is $4\smooth$ smooth for all $t$. Furthermore, the radius of the constraint set in the $t$-th round is $R_t = 2^{T/2}R$. 
Hence, the Lipschitz constant $G_t \leq \lip + 8\smooth R_t \leq O\br{\lip + \smooth 2^{T/2}}$.
Now we instantiate $\alpha_t$,
which is the excess population risk bound of the DP-SCO sub-routine.

\paragraph{Optimal rate:}

The excess population risk guarantee of
Algorithm \ref{alg:noisyGD}
is in Lemma \ref{lem:population-dp-strongly-convex-composite}, with (in context of the notation in the Lemma) Lipschitz parameter
$\lip$ being the same and $G_\Delta = O\br{\smooth2^{T/2}}$. 
Therefore, we  have $\alpha_t = \tilde O\br{\frac{G^2}{\lambda_t n}+\frac{dG^2}{\lambda_tn^2\varepsilon^2}}$.
Plugging in the above estimate, we get, 
\begin{align*}
    \mathbb{E}\norm{\nabla \lossF(\out;\cD)} = \tilde O\br{\frac{G}{\sqrt{n}} + \frac{\sqrt{d}G}{n\varepsilon} + \sqrt{\frac{\lambda}{\smooth}}R} = \tilde O\br{\frac{G}{\sqrt{n}} + \frac{\sqrt{d}G}{n\varepsilon}}
\end{align*}
where the last step follows by setting of $\lambda$. %

The optimality claim follows by combining the non-private lower bound in Theorem \ref{thm:population-convex-gradient-small}, and the DP empirical stationarity lower bound in Theorem \ref{thm:lower-bound-convex} together with a reduction to population stationarity as in \citep[Appendix C]{BFTT19}.

\paragraph{Linear time rate:}
The excess population risk guarantee of
Algorithm \ref{alg:phased-sgd}
is in Lemma
\ref{lem:convex-phased-sgd}, with
Lipschitz parameter
$\lip$ being the same and $G_\Delta = O\br{\smooth2^{T/2}}$.
This gives us $\alpha_t = \tilde O\br{\frac{\lip^2}{\lambda_t n}+\frac{d\lip^2}{\lambda_tn^2\varepsilon^2}}$, and thus 
\begin{align*}
    \mathbb{E}\norm{\nabla \lossF(\out;\cD)} = \tilde O\br{\frac{\lip}{\sqrt{n}} + \frac{\sqrt{d}\lip}{n\varepsilon} +
    \sqrt{\lambda\smooth}R} 
    = \tilde O\br{\frac{\lip}{\sqrt{n}} + \frac{\sqrt{d}\lip}{n\varepsilon} + \frac{\smooth R}{\sqrt{n}}}
\end{align*}
where the last step follows by setting of $\lambda$.
Finally, note that the Lemma \ref{lem:convex-phased-sgd} requires that $n = \tilde \Omega\br{\frac{\smooth+\lambda_t}{\lambda_t}}$ for all $t$. This can be checked to be satisfied by substituting the value of $\lambda_t$.
\end{proof}

\subsection{Utility Lemmas}
We first present some key results which will be useful in the proofs.

\begin{lemma}
\label{lem:radius_bound}
Let $f:\bbR^d \rightarrow \bbR$ be an $\smooth$-smooth convex function and let $w^* = \argmin_{w\in \bbR^d}f(w)$. Let $R = \norm{w^*}$ and $w_0 \in \bbR^d$ such that $\norm{w_0}\leq R$. Define $\tilde f(w) = f(w)+\frac{\lambda}{2} \norm{w-w_0}^2$ and let $\tilde w = \argmin \tilde f(w)$. Then for any $\lambda\geq 0$,  $\norm{\tilde w}\leq \sqrt{2}R$.
\end{lemma}
\begin{proof}
From optimality criterion, $0 = \nabla \tilde f(\tilde w) = \nabla f(\tilde w) + \lambda \br{\tilde w- w_0}$. Therefore, $\nabla f(\tilde w) = \lambda \br{w_0-\tilde w}$ and thus $\ip{\nabla f(\tilde w)}{w_0 - \tilde w}>0$.
Furthermore, since $f$ is convex, from monotonicity,  $\ip{\nabla f(\tilde w)}{w^* - \tilde w} \leq 0$. Since both $w_0$ and $w^*$ lie in the ball of radius $R$ (say $\cW_R$), the above two implies that the hyperplane $H = \bc{w:\ip{\nabla f(\tilde w)}{w-\tilde w}=0}$ intersects with $\cW_R$. Furthermore, since $\nabla f(\tilde w) = \lambda \br{w_0-\tilde w}$, we have that $\tilde w$ is the projection of $w_0$ on $H$ i.e. $\Pi_H(w_0)$. 

Let $w' = \Pi_H(0)$. We have that $w' \in \cW_R$; this is because the hyperplane cuts the hypersphere $\cW_R$ creating a spherical cap and $w'$ is the center of the cap. From properties of convex projections $\norm{\Pi_H(w_0) - \Pi_H(0)} \leq \norm{w_0-0} \leq R$. Furthermore, $\Pi_H(0)$ and $\Pi_H(w_0) -\Pi_H(0)$ are orthogonal. Hence $\norm{\tilde w}^2 = \norm{\Pi_H(w_0)} ^2 = \norm{\Pi_H(0)}^2 + \norm{\Pi_H(w_0) - \Pi_H(0)}^2 \leq 2R^2$.
\end{proof}

We state the following result from 
\cite{allen2018make, foster2019complexity}.

\begin{lemma}
\label{lem:RR-key-lemma}
Suppose for every $t=1,2,\ldots T$, $\mathbb{E}[\lossF^{(t-1)}(\out_t;\cD)] - \lossF^{(t-1)}(w_{t-1}^*;\cD) \leq \alpha_t$ where
$\out_t$ are the iterates produced in the algorithm, 
$w_{t-1}^* = \argmin_{w\in \bbR^d}\lossF^{(t-1)}(w;\cD)$ and $\lambda_t = 2^t\lambda$, we have,
\begin{enumerate}
    \item For every $t\geq 1$, $\mathbb{E}[\norm{\out_t - w_{t-1}^*}^2] \leq \frac{2\alpha_t}{\lambda_{t-1}}$
    \item For every $t\geq 1$, $\mathbb{E}[\norm{\out_t - w_t^*}^2]\leq \frac{\alpha_t}{\lambda_t}$
    \item $\mathbb{E}[\sum_{t=1}^T\lambda_t\norm{\out_t- w_T^*}]\leq 4 \sum_{t=1}^T\sqrt{\alpha_t\lambda_t}$
\end{enumerate}
\end{lemma}

\subsection{Lemmas for NoisyGD (Algorithm \ref{alg:noisyGD})}

\begin{lemma}
\label{lem:population-dp-strongly-convex-composite}
 Consider a function $ f(w;x) =  \ell(w;x) + \Delta(w)$, where $w\mapsto \ell(w;x)$ is convex and $\lip$ Lipschitz for all $x$, and $\Delta(w)$ is $\lambda$ strongly convex, $G_\Delta$ Lipschitz and $H_\Delta$ smooth over a bounded convex set $\cW$.  Algorithm \ref{alg:psgd-output-perturbation} run with parameters
 $\eta  = \frac{\log{T}}{\lambda T}$,
 $\sigma^2 = \frac{64\lip^2T\log{1/\delta}}{n^2\varepsilon^2}$,
 $T=\max\br{\frac{\smooth+H_\Delta}{\lambda}\log{\frac{\smooth+H_\Delta}{\lambda}},\frac{n^2\varepsilon^2 \br{\lip^2+G_\Delta^2}}{d\lip^2\log{1/\delta}}}$ and 
 $\cS(\bc{w_t}_t) = \frac{1}{\sum_{t=1}^T\br{1-\eta \lambda}^{-t}}\sum_{t=1}^T\br{1-\eta \lambda}^{-t}w_t$
 satisfies $(\varepsilon,\delta)$-DP and given a dataset $S$
 of $n$ i.i.d. points from $\cD$,
 the excess population risk of its output $\out$ is bounded by,
\begin{align*}
       \mathbb{E}\left[\lossF(\out;\cD)-\min_{w\in \cW_R}\lossF (w;\cD)\right] = O\br{\frac{\lip^2}{\lambda n}+\frac{d\lip^2\log{1/\delta}}{\lambda n^2\varepsilon^2}}.
\end{align*}
\end{lemma}
\begin{proof}
For the privacy analysis, as in \cite{bassily2014private},
for fixed $w$, the sensitivity of the gradient update is bounded by $\frac{2\lip}{n}$. Applying advanced composition, we have that $\sigma^2 = \frac{64\lip^2T\log{1/\delta}}{n^2\varepsilon^2}$ suffices for $(\varepsilon,\delta)$-DP. 

For utility, we first compute a bound on uniform argument stability of the algorithm; let $\bc{w_t}$ and $\bc{w_t'}$ be sequence of iterates on neighbouring datasets. 
Note that the function $w\mapsto \lossf(w;x)$ is $\smooth+H_\Delta$-smooth and $\lambda$-strongly convex for all $x$.
From the setting of $T$, we have that the step size $\eta \leq \frac{1}{\smooth+H_\Delta}$, hence from the standard stability analysis,

\begin{align*}
    w_{t+1}-w_{t+1}' &= w_t - \eta \nabla  L(w_t;S) - \eta \nabla \Delta(w_t) -w_t' + \eta \nabla  L(w_t';S') + \eta \nabla \Delta(w_t')\\
    & = w_t - w_t' 
    -\eta \br{\nabla  L(w_t;S) + \nabla \Delta(w_t) - \nabla  L(w_t';S)
    -\eta \nabla \Delta(w_t')} \\
    &+\eta\br{\nabla  L(w_t';S') -\nabla  L(w_t';S)} \\
    & = \br{\bbI - \eta\br{\nabla^2  L(\tilde w_t;S) + \nabla^2 \Delta(\tilde w_t)}}\br{w_t - w_t'}\\
       &+\eta\br{\nabla  L(w_t';S') -\nabla  L(w_t';S)} 
\end{align*}
where the last equality follows from Taylor remainder theorem where $\tilde w_t$ is some intermediate point on the line joining $w_t$ and $w_t'$. Using the fact that $\eta \leq \frac{1}{\smooth+H_\Delta}$, we have

\begin{align*}
    \norm{w_{t+1}-w_{t+1}'} &\leq\br{1-\eta \lambda}\norm{w_t-w_t'} + \frac{2\eta \lip}{n} 
     \leq \frac{2\lip}{\lambda n}
\end{align*}
The above gives the same bound for the iterate using the selector $\cS$,
\begin{align*}
    \norm{\cS(\bc{w_t})-\cS(\bc{w_t'})}
     \leq \frac{2\lip}{\lambda n}
\end{align*}
Note that the overall Lipschitz constant for the empirical loss is $\tilde \lip=\lip+G_\Delta$.
For the excess empirical risk guarantee, we use Lemma 5.2 in \cite{feldman2020private} to get,
\begin{align*}
    \mathbb{E}\left[ L\br{\out;S} + \Delta (\out) -  L(w^*;S) - \Delta(w^*)\right] 
    & = \mathbb{E}\left[ \lossF\br{\out;S} -  \lossF(w^*;S)\right]\\
    &= \tilde O\br{\frac{\tilde \lip^2}{\lambda T}}\\
    &=\tilde O\br{\frac{\tilde \lip^2+\sigma^2d}{\lambda T}} \\&= \tilde O\br{\frac{\tilde \lip^2}{\lambda T} + \frac{d\lip^2\log{1/\delta}}{\lambda n^2\varepsilon^2}} \\&=O\br{\frac{d \lip^2\log{1/\delta}}{\lambda n^2\varepsilon^2}}
\end{align*}
where the last step follows from the setting of $T$. For the population risk guarantee, we have,
\begin{align*}
    \mathbb{E}\left[\lossF(\out;\cD) - \lossF(w^*;\cD)\right] &= \mathbb{E}\left[\lossF(\out;\cD) -  \lossF(\out;S)\right] + \mathbb{E}\left[\lossF(\out;\cD) -  \lossF(w^*)\right] \\
    & = \mathbb{E}[L(\out;\cD) -  L(\out;S)] + O\br{\frac{d\lip^2\log{1/\delta}}{\lambda n^2\varepsilon^2}} \\
    & \leq \lip\mathbb{E}\norm{\out-\out'}+ O\br{\frac{d\lip^2\log{1/\delta}}{\lambda n^2\varepsilon^2}} \\
    &= \tilde O\br{\frac{\lip^2}{\lambda n}+\frac{d\lip^2\log{1/\delta}}{\lambda n^2\varepsilon^2}}
\end{align*}
where the inequality follows from Lipschitzness and standard generalization gap to stability argument.
\end{proof}

\subsection{Lemmas for PhasedSGD (Algorithm \ref{alg:phased-sgd})}

The following lemma gives population risk guarantees for strongly convex functions under privacy, in terms of variance of stochastic gradients, as opposed to standard Lipschitzness bounds.
\begin{lemma}[Variance based bound for constant step-size SGD for strongly-convex functions]
\label{lem:convex-variance-based-bound}
Consider a 
function $ \lossf(w;x)$ such that $w\mapsto \lossf(w;x)$ is $\lambda$ strongly convex, $\smooth$ smooth over a convex set $\cW$ for all $x$ and let $\mathbb{E}_x\norm{\nabla \lossf(w;x) - \mathbb{E}_x\nabla \lossf(w;x)}^2 \leq \cV^2$ for all $w\in \cW$.
 Let $\gamma_t = \br{1-\eta \lambda}^{-t}$.
 Given a dataset $S = \bc{x_1,x_2,\ldots, x_n}$  sampled i.i.d from $\cD$ and $\eta \leq \frac{1}{2\beta}$ as input, for any $w\in \cW$, the iterates of Algorithm \ref{alg:psgd-output-perturbation} satisfy

 \begin{align*}
     \mathbb{E}\left[ \frac{1}{\sum_{t=1}^n \gamma_t} \sum_{t=1}^n\gamma_t\lossF(w_t;\cD)\right] - \lossF(w) \leq \frac{\lambda}{e^{\eta \lambda n}-1}\norm{w_0-w}^2 + \eta \cV^2
 \end{align*}

Furthermore, for $n= \Omega\br{\frac{\smooth}{\lambda}\log{\frac{\smooth}{\lambda}}}$, with $\eta  = \frac{\log{n}}{\lambda n}$ and $\cS(\bc{w_t}_t) = \frac{1}{\sum_{t=1}^n\gamma_t}\sum_{t=1}^n\gamma_tw_t$, the excess population risk of $\tilde w = \cS(\bc{w_t}_t)$ satisfies
\begin{align*}
       \mathbb{E}\left[\lossF(\tilde w;\cD)-\min_{w\in \cW}\lossF (w;\cD)\right] = O\br{\frac{{\cV}^2\log{n}}{\lambda n}}
\end{align*}

\end{lemma}
\begin{proof}

An equivalent way to write the update in Algorithm \ref{alg:psgd-output-perturbation} is 
\begin{align*}
    w_{t+1} = \argmin_{w\in \cW}\br{\ip{\nabla \lossf(w_t,x_t)}{w} +\frac{1}{\eta}\norm{w_t-w}^2+ \psi(w)}
\end{align*}
 where $\psi(w) = 0$ if $w \in \cW$, otherwise $\infty$.

Following standard arguments in convex optimization, for any $w \in \cW$, we have
\begin{align*}
    &\lossF(w_{t+1};\cD) - \lossF(w)\\&= \lossF(w_{t+1};\cD)+ \psi(w_{t+1}) - \lossF(w;\cD) - \psi(w)\\
    &\leq \lossF(w_t) + \ip{\nabla \lossF(w_t)}{w_{t+1}-w_t} + \frac{\smooth}{2}\norm{w_{t+1}-w_t}^2 + \psi(w_{t+1}) \\
    &+\lossF(w;\cD) - \psi(w)\\
    & \leq \ip{\nabla \lossF(w_t)}{w_{t+1}-w_t} + \ip{\nabla \lossF(w_t)}{w_t - w} - \frac{\lambda}{2}\norm{w_t-w}^2+\frac{\smooth}{2}\norm{w_{t+1}-w_t}^2 \\
    &+ \psi(w_{t+1}) 
    +\lossF(w;\cD) - \psi(w)\\
    & = \mathbb{E}_{z_t}\left[ \ip{\nabla p(w_t;z_t) - \nabla \lossF(w;\cD)}{w_t - w_{t+1}} + \frac{\smooth}{2}\norm{w_{t+1}-w_t}^2 + \ip{\nabla p(w_t;z_t)}{w_t-w}\right]\\
    & - \frac{\lambda}{2}\norm{w_t-w}^2+\psi(w_{t+1}) +\lossF(w;\cD) - \psi(w)\\
     & \leq \mathbb{E}_{z_t}\Big[ \ip{\nabla p(w_t;z_t) - \nabla \lossF(w;\cD)}{w_t - w_{t+1}} - \br{\frac{1}{2\eta}- \frac{\smooth}{2}}\norm{w_{t+1}-w_t}^2 \\
    & +\br{\frac{1}{2\eta}- \frac{\lambda}{2}}\norm{w_t-w}^2 - \frac{1}{2\eta}\norm{w_{t+1}-w}^2\Big]\\
    & \leq \mathbb{E}_{z_t}\Big[\frac{\eta}{2\br{1-\eta \smooth}}\norm{\nabla p(w_t;z_t) - \nabla \lossF(w;\cD)}^2 +  \br{\frac{1}{2\eta}- \frac{\lambda}{2}}\norm{w_t-w}^2 - \frac{1}{2\eta}\norm{w_{t+1}-w}^2\Big] \\
    & \leq \eta \cV^2 + \mathbb{E}_{z_t}\left[\br{\frac{1}{2\eta}- \frac{\lambda}{2}}\norm{w_t-w}^2 - \frac{1}{2\eta}\norm{w_{t+1}-w}^2\right]
\end{align*}
where the first inequality follows from smoothness, the second from strong convexity, the third from Fact D.1 in \cite{allen2018make}, fourth from AM-GM inequality and the last from the assumption about variance bound on the oracle.

Now, the above is exactly the bound obtained in the proof of Lemma 5.2 in \cite{feldman2020private} with the second moment on gradient norm replaced by variance. Repeating the rest of the arguments in that Lemma gives us the claimed result.
\end{proof}

\begin{lemma}[Privacy of Algorithm \ref{alg:psgd-output-perturbation}]
Consider a function $ \lossf(w;x) = \ell(w;x) + \Delta(w)$ such that 
 $w\mapsto \ell(w;x)$ is convex, $\lip$ Lipschitz, $\smooth$-smooth for all $z$, and $\Delta(\cdot)$ is $\lambda$ strongly convex, $G_\Delta$ Lipschitz and $H_\Delta$ smooth over a bounded set $\cW$. 
For $n = \Omega\br{\frac{\smooth + H_\Delta}{\lambda}\log{\frac{\smooth + H_\Delta}{\lambda}}}$, Algorithm \ref{alg:psgd-output-perturbation} with input as function $(w,x)\mapsto \lossf(w;x)
$, $\sigma^2 = \frac{64G^2\br{\log{n}}^2\log{1/\delta}}{\lambda^2 n^2\varepsilon^2}$, $\eta = \frac{\log{n}}{\lambda n}$ and $\cS\br{\bc{w_t}_{t=1}^n} = \frac{1}{\sum_{t=1}^n \gamma_t} \sum_{t=1}^n\gamma_t w_t$ for any weights $\gamma_t$ satisfies $(\varepsilon,\delta)$-DP.

\end{lemma}
\begin{proof}
We start with computing the sensitivity of the algorithm's output:
let $\bc{w_t}$ and $\bc{w_t'}$ be sequence of iterates produced by Algorithm \ref{alg:psgd-output-perturbation} on
neighbouring datasets. 
Note that the function $w\mapsto \lossf(w;x)$ is $\smooth' = \smooth+H_\Delta$-smooth and $\lambda$-strongly convex for all $x$.
From the assumption on $n$, we have that the step size $\eta \leq \frac{1}{H+H_\Delta}$.
Suppose the differing sample between neighbouring datasets is $x_j$, then
$w_t = w_t'$ for all $t\leq j$. 
Also, 
$$\norm{w_{j+1} - w_{j+1}'} = \eta\norm{\nabla \ell(w_j;x_j) - \nabla \ell(w_j;x_j')} \leq 2\eta \lip = \frac{2\lip\log{n}}{\lambda n}$$
Now, for any $t>j$,
as in the standard stability analysis we have,

\begin{align*}
    w_{t+1}-w_{t+1}' &= w_t - \eta \nabla \ell(w_t;x_t) - \eta \nabla \Delta(w_t) -w_t + \eta \nabla \ell(w_t';x_t) + \eta \nabla \Delta(w_t')\\
    & = \br{\bbI - \eta\br{\nabla^2 \ell(\tilde w_t;x_t) + \nabla^2 \Delta(\tilde w_t)}}\br{w_t - w_t'}
\end{align*}

where the last equality follows from Taylor remainder theorem where $\tilde w_t$ is some intermediate point in the line joining $w_t$ and $w_t'$. Using the fact that $\eta \leq \frac{1}{\smooth+H_\Delta}$ and $\lambda$ strong convexity, we have

\begin{align*}
    \norm{w_{t+1}-w_{t+1}'} &\leq\br{1-\eta \lambda}\norm{w_t-w_t'} \leq \norm{w_{j+1}-w_{j+1}'}
     \leq \frac{2\lip \log{n}}{\lambda n}
\end{align*}
Applying convexity to the weights in the definition of the selector function $\cS$, we get,
\begin{align*}
    \norm{\cS(\bc{w_t})-\cS(\bc{w_t'})}
     \leq \frac{2\lip\log{n}}{\lambda n}
\end{align*}
The privacy proof now follows from the Gaussian mechanism guarantee.
\end{proof}

\begin{lemma}[Phased SGD composite guarantee] 
\label{lem:convex-phased-sgd}
 Consider a function $ \lossf(w;x) = \ell(w;x) + \Delta(w)$
 where $w\mapsto \ell(w;x)$ is convex, $\lip$ Lipschitz, $\smooth$ smooth for all $x$, and $\Delta(w)$ is $\lambda$ strongly convex, $G_\Delta$ Lipschitz and $H_\Delta$ smooth over a bounded set $\cW$. 
For 
$n = \Omega\br{\frac{K\br{\smooth+H_\Delta}}{\lambda}\log{\frac{\smooth + H_\Delta}{\lambda}}}$,  Algorithm \ref{alg:psgd-output-perturbation} with $\sigma^2 = \frac{64\lip^2K^2\br{\log{n}}^2\log{1/\delta}}{\lambda^2 n^2\varepsilon^2}$, 
satisfies $(\varepsilon,\delta)$-DP.
Furthermore, 
with input as function $(w,x)\mapsto \lossf(w;x)$, a dataset $S$ of $n$ samples drawn i.i.d. from $\cD$,  $\eta = \frac{\log{n}}{\lambda n}$,
 $K = \ln{\ln{n}}$,
$\gamma_t = \br{1-\eta\lambda}^{-t}$ and $\cS\br{\bc{w_t}_{t=1}^n} = \frac{1}{\sum_{t=1}^n \gamma_t} \sum_{t=1}^n\gamma_t w_t$,  the excess population risk of output $w_K$ is bounded as
\begin{align*}
    \mathbb{E}\left[\plossD{w_K}\right] - \min_{w\in \cW}\plossD{w} = \tilde O\br{\frac{\lip^2}{\lambda n}+\frac{d\lip^2}{\lambda n^2\varepsilon^2}}
\end{align*}
\end{lemma}

\begin{proof}
The privacy proof simply follows from parallel composition. 
For the utility proof, we repeat the arguments in Theorem 5.3 in \cite{feldman2020private} substituting the variance-based bound from Lemma \ref{lem:convex-variance-based-bound}.
Note that the variance of the stochastic gradients used, $\cV^2\leq \lip^2$, this gives us,
\begin{align*}
    \mathbb{E}\left[\plossD{w_K}\right] - \min_{w\in \cW}\plossD{w} = \tilde O\br{\frac{\lip^2}{\lambda n}+\frac{d\lip^2}{\lambda n^2\varepsilon^2}}
\end{align*}
\end{proof}

%% file: appendix/app-glm.tex
\section{Missing Results for Generalized Linear Models}
\label{app:glm}

We first give the definition of oblivious subspace embedding.
\begin{definition}[$(r,\tau,\beta)$-oblivious subspace embedding]
\label{defn:ose}
A random matrix $\Phi \in \bbR^{k\times d}$ is an $(r,\tau,\beta)$-oblivious subspace embedding if for any $r$ dimensional linear subspace in $\bbR^d$, say $V$, we have that with probability at least $1-\beta$, for all $x \in V$, 
\begin{align*}
    \br{1-\tau}\norm{x}^2 \leq \norm{\Phi x}^2 \leq \br{1+\tau}\norm{x}^2
\end{align*}

\end{definition}
It is well-known that JL matrices with embedding dimension $k = O\br{\frac{r\log{2/\beta}}{\tau^2}}$ are
 $(r,\tau,\beta)$-oblivious subspace embeddings and 
 can be constructed efficiently \cite{cohen2016nearly}. A simple example is a scaled Gaussian random matrix, $\Phi = \frac{1}{\sqrt{k}}\mathbf{G}$ where entries of $\mathbf{G}$ are independent and distributed as $\cN(0,1)$.

\begin{proof}[Proof of Theorem \ref{thm:jl-reduction}]
We first prove privacy.
Let $G(S)$ and $H(S)$ be the bounds on the Lipschitz and smoothness constants of the family of loss functions $\bc{w\mapsto f(w;\Phi x)}_{x\in S}$.
With $k=\Omega(\log{2n/\delta})$, from the JL-property, it follows that with probability at least $1-\delta/2$, $G(S) \leq 2\lip\norm{\cX}$ and $H(S) \leq 2\smooth\norm{\cX}^2$.
Hence, using the fact that $\cA$ is $(\varepsilon,\delta/2)$-DP, we have that Algorithm \ref{alg:jlmethod} is
$(\varepsilon,\delta)$-DP.

We now proceed to the utility part.
Let $\tilde w \in \bbR^k$ be the output of the base algorithm in low dimensions.
Note that the final output is $\out = \Phi^\top \tilde w $. The transpose of the JL matrix can only increase the norm by the polynomial factor of $d$ and $n$, hence $\norm{\out} \leq \text{poly}(n,d)\norm{\tilde w}$.
By assumption, 
$\mathbb{P}\br{\norm{\tilde w} > \text{poly}(n,d,\lip,\smooth)} \leq\frac{1}{\sqrt{n}}$.
Hence we also have that $\bbP\br{\norm{\out} > \text{poly}(n,d,\lip,\smooth)} \leq \frac{1}{\sqrt{n}}$.
Let $\cW \subseteq \bbR^d$ denote the above set with radius $\text{poly}(n,d,\lip,\smooth)$.

We now decompose the population stationarity as,
\begin{align}
\nonumber
    \mathbb{E}\norm{\nabla \lossF(\out;\cD)} &\leq \mathbb{E}\norm{\nabla\lossF(\out;\cD) - \nabla\lossF(\out;S)} + \norm{\nabla \lossF(\out;S)} \\
    &\leq \mathbb{E}\sup_{w\in \cW}\norm{\nabla\lossF(w;\cD) - \nabla\lossF(w;S)} + \frac{\lip\norm{\cX}}{\sqrt{n}} + \mathbb{E}\norm{\nabla \lossF(\out;S)},
    \label{eqn:glm-risk-decomposition}
\end{align}

where the last inequality follows from the above reasoning that 
that $P\br{\out \in \cW} \geq 1-\frac{1}{\sqrt{n}}$. 
The first term is bounded from uniform convergence guarantee in Lemma \ref{lem:uc-glm} noting that the dependence on $\norm{\cW}$ in the Lemma is only poly-logarithmic.
\begin{align}
\label{eqn:glm-uc-bound}
     \mathbb{E}\sup_{w\in \cW}\norm{\nabla\lossF(w;\cD) - \nabla\lossF(w;S)} = \tilde O\br{\frac{\lip\norm{\cX}}{\sqrt{n}}}
\end{align}
We now prove a bound on the empirical stationarity.
Note that it suffices to prove a high-probability (over the random JL matrix) bound because the norm of gradient is bounded in worst case by $\lip\norm{\cX}$. Thus the expected norm of gradient of the output is bounded by the high probability bound by considering a small enough failure probability.

From the assumption on $\cA$,
with probability at least $1-\delta/2$,
\begin{align*}
     \norm{\nabla  \lossF(\tilde w;\Phi S)} = \mathbb{E}\norm{\frac{1}{n}\sum_{i=1}^n \phi_{y_i}'(\ip{\tilde w}{\Phi x_i})\Phi x_i} \leq g(k,n, 2\lip\norm{\cX}, 2\lip\norm{\cX}, \varepsilon, \delta/2)
\end{align*}

We now use the fact that if 
$k = O\br{\mathsf{rank}\log{2n/\delta}}$, then 
the JL transform is an $(\rank, 1/2,\delta/2)$ oblivious subspace embedding (see Definition \ref{defn:ose}).
Thus, it approximates the norm of any vector in $\mathsf{span}(\bc{x_i}_{i=1}^n)$, and hence any gradient. Therefore,
\begin{align*}
    \mathbb{E}\norm{\nabla \lossF(\tilde w;\Phi S)} &= \mathbb{E}\norm{\Phi\br{\frac{1}{n}\sum_{i=1}^n \phi_{y_i}'(\ip{\tilde w}{\Phi x_i}) x_i}} \geq \br{1-\sqrt{\frac{\mathsf{rank}}{k}}}\mathbb{E}\norm{\frac{1}{n}\sum_{i=1}^n \phi_{y_i}'(\ip{\tilde w}{\Phi x_i}) x_i} \\
    & \geq \frac{1}{2}\mathbb{E}\norm{\frac{1}{n}\sum_{i=1}^n \phi_{y_i}'(\ip{\tilde w}{\Phi x_i}) x_i} = \frac{1}{2}\mathbb{E}\norm{\frac{1}{n}\sum_{i=1}^n \phi_{y_i}'(\ip{\Phi^\top\tilde w}{x_i}) x_i} = \frac{1}{2}\mathbb{E}\norm{\nabla \lossF(\out; S)}
\end{align*}

Thus with $k= O\br{\mathsf{rank}\log{2n/\delta}}$, we get 

\begin{align*}
\mathbb{E}\norm{\nabla \lossF(\out; S)} \leq
g(k,n, 2\lip\norm{\cX},2\smooth\norm{\cX}^2, \varepsilon,\delta) = g(\mathsf{rank}, n,2\lip\norm{\cX},2\smooth\norm{\cX}^2,\varepsilon,\delta)
\end{align*}

For the other bound, let $I_{d-k} \in \bbR^{d\times k}$ denote the matrix with first $k$ diagonal entries, $\br{I_{d-k}}_{j,j}$ with $j\in [k]$, are $1$ and the rest of the matrix is zero.
We have,
\begin{align*}
    &\mathbb{E}\norm{\nabla \lossF(\out; S)} \\
    &= \mathbb{E}\norm{\frac{1}{n}\sum_{i=1}^n \phi_{y_i}'(\ip{\Phi^\top\tilde w}{x_i}) x_i} \\
    &\leq \mathbb{E}\norm{\frac{1}{n}\sum_{i=1}^n \phi_{y_i}'(\ip{\tilde w}{\Phi x_i}) I_{d-k}\Phi x_i} +
    \mathbb{E}\left[\norm{\frac{1}{n}\sum_{i=1}^n \phi_{y_i}'(\ip{\tilde w}{\Phi x_i}) x_i
    -
    \frac{1}{n}\sum_{i=1}^n \phi_{y_i}'(\ip{\tilde w}{\Phi x_i}) I_{d-k}\Phi x_i} \right] \\
       &\leq \mathbb{E}\norm{I_{d-k}}\norm{\frac{1}{n}\sum_{i=1}^n \phi_{y_i}'(\ip{\tilde w}{\Phi x_i}) \Phi x_i} +
     \frac{1}{n}\mathbb{E}\sum_{i=1}^n \abs{\phi_{y_i}'(\ip{\tilde w}{\Phi x_i})} \abs{\norm{x_i - I_{d-k}\Phi x_i}}\\
    & \leq \mathbb{E}\norm{\nabla \lossF(\tilde w;\Phi S)} + \frac{1}{n}\mathbb{E}\sum_{i=1}^n \lip \norm{I-I_{d-k}\Phi}\norm{x_i}
    \\
      & \leq g(k,n,2\lip\norm{\cX},2\smooth\norm{\cX}^2,\varepsilon,\delta/2) +
      \lip\norm{\cX}\mathbb{E}
      \norm{I-\mathbf{H}}
\end{align*}
where the second inequality follows from triangle inequality, the third inequality follows from $\lip$-Lipschitzness of the GLM,
the third inequality follows from the accuracy guarantee of the base algorithm and substituting $ \mathbf{H} = I_{d-k}\Phi$.
To bound $\mathbb{E}\norm{I-\mathbf{H}}$, we use concentration properties of distribution used in the construction of JL matrices. 
Specifically, using the scaled Gaussian matrix construction, from concentration of extreme eignevalues of square Gaussian matrices, we have that $\mathbb{E}\norm{I-\mathbf{H}} = \tilde O\br{\frac{1}{\sqrt{k}}}$ \cite{rudelson2010non}.
This gives us,
\begin{align*}
    \mathbb{E}\norm{\nabla \lossF(\out; S)}   %
    & \leq g(k,n,2\lip\norm{\cX},2\smooth\norm{\cX}^2,\varepsilon,\delta/2) + \tilde O\br{\frac{\lip\norm{\cX}}{\sqrt{k}}}
\end{align*}
Choosing $k$ to minimize the above yields the bound of $\tilde O\br{\frac{\lip\norm{\cX}}{\sqrt{k}}}$. Combining the two cases, yields the bound of $g(k,n,2\lip\norm{\cX},2\smooth\norm{\cX}^2,\varepsilon,\delta/2)$
on gradient norm. Plugging this and the bound in Eqn. \eqref{eqn:glm-uc-bound} in Inequality \eqref{eqn:glm-risk-decomposition} gives the claimed bound.
\end{proof}

\begin{lemma}
\label{lem:uc-glm}
Let $\cD$ be a probability distribution over $\cX$ such that $\norm{x}\leq \norm{\cX}$ for all $x \in \text{supp}(\cD)$.
Let $\lossf(w;(x,y)) = \phi_y\br{\ip{w}{x}}$ be an $\smooth$-smooth $\lip$-Lipschitz GLM.
Then, with probability at least $1-\beta$, over a draw of $n$ i.i.d. samples $S$ from $\cD$, we have
\begin{align*}
    \sup_{w\in \cW}\norm{\nabla \lossF(w;\cD) - \nabla \lossF(w;S)} \leq\frac{4\lip\norm{\cX}\log{2n^{3/2}\norm{\cW}\smooth\norm{\cX}/\lip}}{\sqrt{n}} + \frac{4\lip \norm{\cX}\sqrt{\log{1/\beta}}}{\sqrt{n}}
\end{align*}
\end{lemma}
\begin{proof}
We first give a bound on the expected uniform deviation,  $\mathbb{E}_{S\sim \cD^n} \sup_{w\in \cW}\norm{\nabla \lossF(w;\cD) - \nabla \lossF(w;S)}$.
The gradient of the loss function is $\nabla \lossf(w;x) = \phi'_x\br{\ip{w}{x}}x$. We start with the standard symmetrization trick,
\begin{align}
\nonumber
   &\mathbb{E}_{S\sim \cD^n}\sup_{w\in \cW} \norm{\nabla \lossF(w;\cD) - \nabla \lossF(w;S)}\\
   \nonumber
   &= \mathbb{E}_{S\sim \cD^n}\sup_{w\in \cW} \norm{\mathbb{E}\phi_y'\br{\ip{w}{x}}x - \frac{1}{n}\sum_{i=1}^n\phi'_{x_i}\br{\ip{w}{x_i}}x_i} \\
   \nonumber
    &= \mathbb{E}_{S\sim \cD^n}\sup_{w\in \cW} \norm{\mathbb{E}_{\bc{x_i'}\sim \cD^n}\frac{1}{n}\sum_{i=1}^n\phi_{y_i'}'\br{\ip{w}{x_i'}}x_i' - \frac{1}{n}\sum_{i=1}^n\phi'_{x_i}\br{\ip{w}{x_i}}x_i}  \\
    \nonumber
    & \leq \mathbb{E}_{S,S'\sim \cD^n}\sup_{w\in \cW}\norm{\frac{1}{n}\sum_{i=1}^n\phi_{y_i'}'\br{\ip{w}{x_i'}}x_i' - \frac{1}{n}\sum_{i=1}^n\phi'_{x_i}\br{\ip{w}{x_i}}x_i} \\
    \nonumber
    &= \mathbb{E}_{S,S'\sim \cD^n} 
    \mathbb{E}_{\bc{\sigma_i}}
    \sup_{w\in \cW}\norm{\frac{1}{n}\sum_{i=1}^n\sigma_i\br{\phi_{y_i'}'\br{\ip{w}{x_i'}}x_i' - \phi'_{x_i}\br{\ip{w}{x_i}}x_i}} \\
    \label{eqn:uc-glm-one}
    &\leq 2\mathbb{E}_{S\sim \cD^n}\mathbb{E}_{\bc{\sigma_i}}  \sup_{w\in \cW}\norm{\frac{1}{n}\sum_{i=1}^n\sigma_i\phi_{y_i}'\br{\ip{w}{x_i}}x_i}
\end{align}
where $\sigma_i$ are i.i.d. Rademacher random variables.
For fixed $\bc{x_i}_{i=1}^n$, consider a set $\cW_0$ s.t. for all $w \in \cW$ and $i\in [n]$, there exists $w_0 \in \cW_0$ such that %
$\abs{\ip{w}{x_i} - \ip{w_0}{x_i}} \leq \tau$.
Since $\norm{w} \leq \norm{\cW}$ and $\norm{x_i} \leq \norm{\cX}$, we require only $\frac{2n\norm{\cW}\norm{\cX}}{\tau}$ points in $\cW_0$ to satisfy the above covering condition. Therefore, 
\begin{align}
\nonumber
    &\mathbb{E}_{S\sim \cD^n}\mathbb{E}_{\bc{\sigma_i}}  \sup_{w\in \cW}\norm{\frac{1}{n}\sum_{i=1}^n\sigma_i\phi_{y_i}'\br{\ip{w}{x_i}}x_i} \\
    \nonumber
    &= \mathbb{E}_{S\sim \cD^n}\mathbb{E}_{\bc{\sigma_i}}  \sup_{w\in \cW, w_0 \in \cW_0}\norm{\frac{1}{n}\sum_{i=1}^n\sigma_i\br{\phi_{y_i}'\br{\ip{w}{x_i}}-\phi_{y_i}'\br{\ip{w_0}{x_i}}+\phi_{y_i}'\br{\ip{w_0}{x_i}}}x_i} \\
    \nonumber
    &\leq  \mathbb{E}_{S\sim \cD^n}\mathbb{E}_{\bc{\sigma_i}} \sup_{w\in \cW, w_0 \in \cW_0}\norm{\frac{1}{n}\sum_{i=1}^n\sigma_i\br{\phi_{y_i}'\br{\ip{w}{x_i}}-\phi_{y_i}'\br{\ip{w_0}{x_i}}}x_i}+\norm{\frac{1}{n}\sum_{i=1}^n\sigma_i\phi_{y_i}'\br{\ip{w_0}{x_i}}x_i} \\
    \nonumber
    &\leq \mathbb{E}_{S\sim \cD^n}\mathbb{E}_{\bc{\sigma_i}}\sup_{w\in \cW, w_0\in \cW_0} \smooth\abs{\ip{w}{x_i}-\ip{w_0}{x_i}}\norm{\cX} + \mathbb{E}_{S\sim \cD^n}\mathbb{E}_{\bc{\sigma_i}}\sup_{w_0\in \cW_0}\norm{\frac{1}{n}\sum_{i=1}^n\sigma_i\phi_{y_i}'\br{\ip{w_0}{x_i}}x_i} \\
    \label{eqn:uc-glm-two}
    &\leq \smooth\tau \norm{\cX} + \mathbb{E}_{S\sim \cD^n}\mathbb{E}_{\bc{\sigma_i}}\sup_{w_0\in \cW_0}\norm{\frac{1}{n}\sum_{i=1}^n\sigma_i\phi_{y_i}'\br{\ip{w_0}{x_i}}x_i} 
\end{align}
where the second last inequality follows from smoothness and the last from the definition of cover $\cW_0$. For fixed $w_0$, from standard manipulations, we have,
\begin{align*}
    \mathbb{E}_{\bc{\sigma_i}}\norm{\frac{1}{n}\sum_{i=1}^n\sigma_i\phi_{y_i}'\br{\ip{w_0}{x_i}}x_i}  
    &\leq \sqrt{\mathbb{E}_{\bc{\sigma_i}}\norm{\frac{1}{n}\sum_{i=1}^n\sigma_i\phi_{y_i}'\br{\ip{w_0}{x_i}}x_i}^2}\\
    & = \sqrt{\frac{1}{n^2}\mathbb{E}_{\bc{\sigma_i}}\sum_{i=1}^n\norm{\sigma_i\phi_{y_i}'\br{\ip{w_0}{x_i}}x_i}^2                                                                                                                                                                                                                                                                                                                                                                                                                                                                                                                                                                                                                                                                                                                                                                                                                                                                                                }\\
    &\leq \frac{\lip\norm{\cX}}{\sqrt{n}}
\end{align*}

Using Massart's finite class lemma to handle all $w_0 \in \cW_0$, and substituting the above in Eqn. \eqref{eqn:uc-glm-two}, we get,
\begin{align*}
     &\mathbb{E}_{S\sim \cD^n}\mathbb{E}_{\bc{\sigma_i}}  \sup_{w\in \cW}\norm{\frac{1}{n}\sum_{i=1}^n\sigma_i\phi_{y_i}'\br{\ip{w}{x_i}}x_i}  \leq \smooth\tau \norm{\cX} + \frac{G\norm{\cX}\log{2n\norm{\cW}\norm{\cX}/\tau}}{\sqrt{n}}
\end{align*}
Choosing $\tau = \frac{\lip}{\smooth\sqrt{n}}$, we get, 
\begin{align*}
    &\mathbb{E}_{S\sim \cD^n}\mathbb{E}_{\bc{\sigma_i}}  \sup_{w\in \cW}\norm{\frac{1}{n}\sum_{i=1}^n\sigma_i\phi_{y_i}'\br{\ip{w}{x_i}}x_i}  \leq \frac{2\lip\norm{\cX}\log{2n^{3/2}\norm{\cW}\smooth\norm{\cX}/\lip}}{\sqrt{n}}
\end{align*}
Finally, substituting the above in Eqn. \eqref{eqn:uc-glm-one} gives us the following in-expectation bound.
\begin{align*}
    \mathbb{E}_{S\sim \cD^n} \sup_{w\in \cW}\norm{\nabla \lossF(w;\cD) - \nabla \lossF(w;S)} \leq \frac{4\lip\norm{\cX}\log{2n^{3/2}\norm{\cW}\smooth\norm{\cX}/\lip}}{\sqrt{n}}
\end{align*}

For the high-probability bound, let $\psi(S) = \sup_{w\in \cW}\norm{\nabla \lossF(w;\cD) - \nabla \lossF(w;S)}$ and let $w^* \in \cW$ achieves the supremum. We can bound the increment between neighbouring datasets $S$ and $S'$ as,
\begin{align*}
    \abs{\psi(S) - \psi(S')} &\leq \abs{\norm{\nabla \lossF(w^*;\cD) - \nabla \lossF(w^*;S)} - \norm{\nabla \lossF(w^*;\cD) - \nabla \lossF(w^*;S')}} \\
   &\leq  \norm{\nabla \lossF(w^*;S) -\nabla \lossF(w^*;S')} \\
   & \leq \frac{2\lip\norm{\cX}}{n}
\end{align*}
Finally, applying McDiarmid's inequality gives the claimed bound.
\end{proof}

\begin{proof}[Proof of Corollary \ref{cor:glm-corollary}]
The results follow from Theorem \ref{thm:jl-reduction} provided we show that the conditions on the base algorithm in the Theorem statement are satisfied. The privacy and accuracy claims follow from Theorem \ref{thm:accuracy_tree_private_Spider} and \ref{thm:population-convex-gradient-small} respectively. We note that even though we are given population stationarity guarantee for the convex case, the same bound for empirical stationarity guarantee simply follows from the re-sampling argument in \cite{BFTT19}. The only thing left to show is the high-probability bound on the trajectory of the algorithm.

\paragraph{Non-convex setting with Private Spiderboost:}
From the update in Algorithm \ref{alg:spider}, we have that for any $t$
\begin{align*}
    \norm{\nabla_t} \leq &\sum_{i=1}^{t}\norm{\Delta_i} + \norm{\sum_{i=1}^t{g_t}} 
    \leq 2t\lip + \norm{\sum_{i=1}^t{g_t}}
\end{align*}
where the last inequality follows from the Lipschitzness assumption. Note that $g_t \sim \cN(0,\sigma_t^2\bbI)$ where $\sigma_t \leq O\br{
\max{(\sigma_1,\hat\sigma_2)}} = O\br{\text{poly}(n,d,\lip,\smooth)}$. Hence $\norm{\sum_{i=1}^t{g_t}} \leq \sqrt{d \log{1/\beta'}}O\br{\text{poly}(n,d,\lip,\smooth)}$ with probability at least $1-\beta'$.
Taking a union bound over all $t \in T$ gives us $\norm{w_t} \leq \text{poly}(n,d,\lip,\smooth,\log{\text{poly}(n,d)/\beta})$ with probability at least $1-\beta$. Substituting $\beta= \frac{1}{\sqrt{n}}$ yields the guarantee of Theorem \ref{thm:jl-reduction}. \rnote{made a minor rephrasing}

\paragraph{Convex setting with Recursive Regularization:}
Since the iterates are restricted to the constraint set, the final output, with probability one, lies in the set of radius $$R_T = 2^{T/2}\norm{w^*} = O\br{\sqrt{\frac{\smooth}{\lambda}}\norm{w^*}} = O\br{\frac{\smooth \norm{w^*}^{3/2}n}{\lip}}$$ 

which completes the proof.
\end{proof}